\documentclass[journal]{IEEEtran}

\usepackage{algorithm}
\usepackage[noend]{algpseudocode}
\usepackage{amsmath}
\usepackage{amsthm}
\usepackage{amssymb}
\usepackage{color}
\usepackage{graphicx}
\usepackage[hidelinks]{hyperref}
\usepackage[caption=false]{subfig}
\usepackage{multirow}
\usepackage{cite}

\DeclareMathOperator*{\argmin}{\arg\!\min}
\def \func{f} 
\def \RFE {g} 
\def \RFERV {G} 
\def \input {{\mathbf x}} 
\def \inputRV {{\mathbf X}} 
\def \d {d} 
\def \inputspace {\mathcal X} 
\def \D {D} 
\def \ampl {c} 
\def \Ampl {C} 
\def \amplTH {{\bar \ampl}} 
\def \amplLS {{\ampl_N}} 
\def \amplLSRV {{\Ampl_N}} 
\def \k {k} 
\def \freq {{\boldsymbol\omega}} 
\def \freqRV {{\boldsymbol\Omega}} 
\def \b {b} 
\def \B {B} 
\def \reg {\lambda} 
\def \offampl {q} 
\def \A {\mathbf{A}} 
\def \a {\mathbf{a}} 
\def \perturb {\zeta}  

\def \etaRV {{H}} 

\def \iter{n} 
\def \N{N} 

\newtheorem{definition}{Definition}
\newtheorem{theorem}{Theorem}

\newcommand{\email}[1]{\protect\href{mailto:#1}{#1}}

\allowdisplaybreaks

\title{Online Optimization with Costly and Noisy Measurements using Random Fourier Expansions
}

\begin{document}

\author{Laurens~Bliek\IEEEauthorrefmark{1},~
        Hans R. G. W. Verstraete\IEEEauthorrefmark{1},~
~Michel~Verhaegen,~\IEEEmembership{Member,~IEEE}
and Sander Wahls,~\IEEEmembership{Member,~IEEE}%
\thanks{\IEEEauthorrefmark{1}Both authors contributed equally to this work. Corresponding authors: 
 \email{l.bliek@tudelft.nl}, \email{h.r.g.w.verstraete@tudelft.nl}. }%
\thanks{All authors are with the Delft Center for Systems and Control, Delft University of
 Technology, Mekelweg 2, 2628 CD, Delft, Netherlands.}
}

\markboth{IEEE Transactions on Neural Networks and Learning Systems,~Vol.~XX, No.~X, September~20XX}%
{Shell \MakeLowercase{\textit{et al.}}: Bare Demo of IEEEtran.cls for Journals}

\maketitle

\begin{abstract}
This paper analyzes DONE, an online optimization algorithm that iteratively minimizes an unknown function based on costly and noisy measurements.
The algorithm maintains a surrogate of the unknown function in the form of a random Fourier expansion (RFE).
The surrogate is updated whenever a new measurement is available, and then used to determine the next measurement point.
The algorithm is comparable to Bayesian optimization algorithms, but its computational complexity per iteration does not depend on the number of measurements.
We derive several theoretical results that provide insight on how the hyper-parameters of the algorithm should be chosen.
The algorithm is compared to a Bayesian optimization algorithm for an analytic benchmark problem and three applications, namely, optical coherence tomography, optical beam-forming network tuning, and robot arm control.
It is found that the DONE algorithm is significantly faster than Bayesian optimization in the discussed problems, while achieving a similar or better performance.
\end{abstract}





\IEEEpeerreviewmaketitle

\section{Introduction}

\IEEEPARstart{M}{any}
optimization algorithms use the derivative of an objective function, but often this information is not available in practice.
Regularly, a closed form expression for the objective function is not available and function evaluations are costly.
Examples are objective functions that rely on the outcome of a simulation or an experiment.
Approximating derivatives with finite differences is costly in high-dimensional problems, especially if the objective function is costly to evaluate.
More efficient algorithms for derivative-free optimization (DFO) problems exist.
Typically, in DFO algorithms a model is used that can be optimized without making use of the derivative of the underlying function~\cite{conn2009introduction,rios2013derivative}.
Some examples of commonly used DFO algorithms are the simplex method ~\cite{nelder1965simplex}, NEWUOA~\cite{powell2006newuoa}, BOBYQA~\cite{powell2009bobyqa}, and DIRECT~\cite{jones1993lipschitzian}.   
Additionally, measurements of a practical problem are usually corrupted by noise.
Several techniques have been developed to cope with a higher noise level and make better use of the expensive objective functions evaluations. 
Filtering and pattern search optimization algorithms such as implicit filtering~\cite{gilmore1995implicit} and SID-PSM \cite{custodio2007} can handle local minima resulting from high frequency components.
Bayesian optimization, also known as sequential Kriging optimization, deals with heteroscedastic noise and perturbations very well. 
One of the first and best known Bayesian optimization algorithms is EGO~\cite{jones1998efficient}.
Bayesian optimization relies on a surrogate model that represents a probability distribution of the unknown function under noise, for example Gaussian processes or Student's-t processes \cite{krige1951,bergstra2011algorithms,hutter2011sequential,martinez2014}.
In these processes different kernels and kernel learning methods are used for the covariance function \cite{roustant2012dicekriging,snoek2012practical}.
The surrogate model is used to decide where the next measurement should be taken.
New measurements are used to update the surrogate model.
Bayesian optimization has been successfully used in various applications, including active user modeling and reinforcement learning \cite{brochu2010tutorial}, robotics \cite{Martinez2009}, hyper-parameter tuning \cite{bergstra2011algorithms}, and optics \cite{Rehman2015}.

Recently, the Data-based Online Nonlinear Extremum-seeker (DONE) algorithm was proposed in~\cite{Verstraete15}. 
It is similar to Bayesian optimization, but simpler and faster.
The DONE algorithm uses random Fourier expansions~\cite{rahimi2007} (RFEs) as a surrogate model.
The nature of the DONE algorithm makes the understanding of the hyper-parameters easier.
In RFE models certain parameters are chosen randomly.
In this paper, we derive a close-to-optimal probability distribution for some of these parameters.
We also derive an upper bound for the regularization parameter used in the training of the RFE model.

The advantages of the DONE algorithm are illustrated in an analytic benchmark problem and three applications. 
We numerically compare DONE to BayesOpt~\cite{martinez2014}, a Bayesian optimization library that was shown to outperform many other similar libraries in~\cite{martinez2014}.
The first application is optical coherence tomography (OCT), a 3D imaging method based on interference often used to image the human retina \cite{nasiri2009,Bonora13,Verstraete15}.
The second application we consider is the tuning of an optical beam-forming network (OBFN). OBFNs are used in wireless communication systems to steer phased array antennas in the desired direction by making use of positive interference of synchronized signals~\cite{hansen2009phased,roeloffzen2005ring,meijerink2010novel,zhuang2006single,zhuang2010ring,Bliek2015166}.
The third application is a robot arm of which the tip has to be directed to a desired position~\cite{de2009method}.

This paper is organized as follows.
Section~\ref{sec:RFE} gives a short overview and provides new theoretical insights on random Fourier expansions, the surrogate model on which the DONE algorithm is based.
We have noticed a gap in the literature, where approximation guarantuees are given for ideal, but unknown RFE weights, while in practice RFE weights are computed via linear least squares.
We investigate several properties of the ideal weights and combine these results with existing knowledge of RFEs to obtain approximation guarantees for least-square weights.
Section~\ref{sec:algstep1} explains the DONE algorithm.
Theoretically optimal as well as more practical ways to choose the hyper-parameters of this algorithm are given in Section~\ref{sec:hyp}. 
In Section~\ref{sec:appl} the DONE algorithm and BayesOpt are compared for a benchmark problem and for the three aforementioned applications.
We conclude the paper in Section~\ref{sec:conclusion}.

\section{Random Fourier Expansions}\label{sec:RFE}

In this section, we will describe the surrogate model that we will use for optimization.
There is a plethora of black-box modeling techniques to approximate a function from measurements available in the literature, with neural networks, kernel methods, and of course classic linear models probably being the most popular~\cite{hofmann2008kernel,suykens2012nonlinear,theodoridis2015machine}. 
In this paper, we use random Fourier expansions (RFEs) \cite{rahimi2007} to model the unknown function because they offer a unique mix of computational efficiency, theoretical guarantees and ease of use that make them ideal for online processing. While general neural networks are more expressive than random Fourier features, they are difficult to use and come without  theoretical guarantees. Standard kernel methods suffer from high computational complexity because the number of kernels equals the number of measurements. RFEs have been originally introduced to reduce the computational burden that comes with kernel methods, as will be explained next \cite{rahimi2007,rahimi2009weighted,singh2012online}.

Assume that we are provided $N$ scalar measurements $y_i$ taken at measurement points $\input_i\in \mathbb{R}^d$ as well as a kernel $k(\input_i,\input_j)$ that, in a certain sense, measures the closeness of two measurement points. To train the kernel expansion
\begin{align}\label{eq:kernel}
\RFE_{KM}(\input) & = \sum_{i=1}^{N}a_{i} k(\input,\input_i),
\end{align}
a linear system involving the kernel matrix $[k(\mathbf x_i, \mathbf x_j)]_{i,j}$ has to be solved for the coefficients $a_i$. The computational costs of training and evaluating \eqref{eq:kernel} grow cubicly and linearly in the number of datapoints $N$, respectively.
This can be prohibitive for large values of $N$.
We now explain how RFEs can be used to reduce the complexity~\cite{rahimi2007}.
Assuming the kernel $k$ is shift-invariant and has Fourier transform $p$, it can be normalized such that $p$ is a probability distribution~\cite{rahimi2007}.
That is, we have
\begin{align}
k(\mathbf x_i - \mathbf x_j) & = \int_{\mathbb R^{\d}} p(\freq) e^{-i\freq^T(\mathbf x_i - \mathbf x_j)}d\freq.\\
\intertext{We will use several trigonometric properties and the fact that $k$ is real to continue the derivation. This gives}
k(\mathbf x_i - \mathbf x_j) & = \int_{\mathbb R^{\d}} p(\freq) \cos(\freq^T(\mathbf x_i - \mathbf x_j))d\freq\nonumber\\
& =  \int_{\mathbb R^{\d}} p(\freq)   \cos(\freq^T(\mathbf x_i - \mathbf x_j)) \nonumber\\ 
& \quad + p(\freq)\int_0^{2\pi} \cos(\freq^T(\mathbf x_i + \mathbf x_j) + 2\b) d\b d\freq\nonumber\\
& = \frac{1}{2\pi}  \int_{\mathbb R^{\d}} p(\freq)  \int_0^{2\pi} \cos(\freq^T(\mathbf x_i - \mathbf x_j)) \nonumber\\ 
& \qquad + \cos(\freq^T(\mathbf x_i + \mathbf x_j) + 2\b) d\b d\freq\nonumber\\
& = \frac{1}{2\pi}  \int_{\mathbb R^{\d}} p(\freq)\int_0^{2\pi} 2\cos(\freq^T \mathbf x_i + \b) \nonumber\\
& \qquad \cdot \cos(\freq^T\mathbf x_j + \b) d\b d\freq\nonumber\\
& = \mathbb E[2\cos(\freqRV^T \mathbf x_i + \B)\cos(\freqRV^T\mathbf x_j + \B)] \nonumber\\
& \approx \frac{2}{\D}\sum_{\k=1}^{\D}\cos(\freq_{\k}^T \mathbf x_i + \b_{\k})\cos(\freq_{\k}^T\mathbf x_j + \b_{\k}),
\label{eq:kerneltoRFE}
\end{align}
if $\freq_{\k}$ are independent samples of the random variable $\freqRV$ with probability distribution function (p.d.f.) $p$, and $\b_{\k}\in[0,2\pi]$ are independent samples of the random variable $\B$ with a uniform distribution. 
For $\ampl_{\k} = \sum_{i=1}^\N\frac{2}{\D}a_i\cos(\freq_{\k}^T \mathbf x_i + \b_{\k})$ we thus have:
\begin{align}
\RFE_{KM}(\input) &\approx \sum_{\k=1}^{\D} \ampl_{\k} \cos(\freq_{\k}^T\mathbf x + \b_{\k}).
\end{align}
Note that the number of coefficients $\D$ is now independent of the number of measurements $\N$.
This is especially advantageous in online applications where the number of measurements $\N$ keeps increasing.
We use the following definition of a random Fourier expansion.
\\
\begin{definition}\label{def:RFE}

A Random Fourier Expansion (RFE) is a function of the form $\RFE:\mathbb R^{\d} \rightarrow \mathbb R$, 
\begin{align}
\RFE(\input) = \sum_{\k=1}^\D \ampl_{\k} \cos(\freq_{\k}^T \input + \b_{\k}),
\end{align}
with $D\in \mathbb N$, the $\b_{\k}$ being realizations of independent and identically distributed (i.i.d.) uniformly distributed random variables $\B_k$ on $[0,2\pi]$, and with the $\freq_{\k}\in\mathbb R^{\d}$ being realizations of i.i.d. random vectors $\freqRV_k$ with an arbitrary continuous p.d.f. $p_{\freqRV}$. 
The $\B_{\k}$ and the $\freqRV_{\k}$ are assumed to be mutually independent. 
\end{definition}

We finally remark that there are other approaches to reduce the complexity of kernel methods and make them suitable for online processing, which are mainly based on sparsity~\cite{burges1996simplified,scholkopf2002sampling,quinonero2005unifying,chen2013quantized}. However, these are much more difficult to tune than using RFEs~\cite{singh2012online}. 
It is also possible to use other basis functions instead of the cosine, but the cosine was among the top performers in an exhaustive comparison with similar models~\cite{zhang2015comprehensive}.
Moreover, the parameters of the cosines have intuitive interpretations in terms of the Fourier transform. 

\subsection{Ideal RFE Weights}\label{sec:theory}

In this section, we deal with the problem of fitting a RFE to a given function $f$. 
We derive ideal but in practice unknown weights $\ampl$.
We start with the case of infinitely many samples and basis functions (see also~\cite{girosi1992convergence, barron1993universal}), which corresponds to turning the corresponding sums into integrals.
\begin{theorem}\label{thm:zeroerror}

Let $\func \in L^2(\mathbb R^{\d})$ be a real-valued function and let 
\begin{align} \label{eq:c_as_cos}
\amplTH(\freq,\b)& =\left\{\begin{array}{cc}\frac{1}{\pi}|\hat \func(\freq)|\cos(\angle \hat \func(\freq) - \b), & \ \b\in[0,2\pi],\\
0, & \mathrm{otherwise}.\end{array}\right.
\end{align}
Then, for all $\input\in \mathbb R^{\d}$,
\begin{align}
	\func(\input) & = \frac{1}{(2\pi)^{\d}} \int_{\mathbb R^{\d}}\int_0^{2\pi}\amplTH(\freq,\b)\cos(\freq^T \input + \b)d\b d\freq. \label{eq:zeroerror}
\end{align}    
\end{theorem}
Here, $|\hat \func|$ and $\angle \hat \func$ denote the magnitude and phase of the Fourier transform $\hat \func(\freq) = \int_{\mathbb R^{\d}} \func(\input)e^{-i \freq^T \input}d\input$. 
The sets $L^2$ and $L^\infty$ denote the space of square integrable functions and the space of all essentially bounded functions, respectively.
\begin{proof}
For $\b\in[0,2\pi]$, we have
\begin{align}
\amplTH(\freq,\b) &= \frac{1}{\pi}|\hat \func(\freq)|\cos(\angle \hat \func(\freq) - \b) \nonumber \\
&=\frac{1}{\pi}\mathrm{Re}\left\{\hat \func(\freq) e^{-ib}\right\}. \label{eq:c_as_real}
\end{align}
Using that $\func(\input)$ is real, we find that 
\begin{align}    
 \func(\input) = & \mathrm{Re}\left\{ \frac{1}{(2\pi)^{\d}} \int_{\mathbb R^{\d}}\hat \func(\freq)e^{i\freq^T \input}d\freq  \right\}\nonumber\\
= & \mathrm{Re}\bigg\{ \frac{1}{(2\pi)^{\d}} \int_{\mathbb R^{\d}}\Big(\hat \func(\freq)e^{i\freq^T \input} \frac{1}{2\pi} \int_0^{2\pi}1 d\b + 
    \nonumber\\ & 
    \hat \func(\freq)e^{-i\freq^T \input} \underbrace{\int_0^{2\pi} e^{-2i\b} d\b}_{=0}\Big) d\freq \bigg\}\nonumber\\
= & \mathrm{Re}\left\{ \frac{1}{\pi} \frac{1}{(2\pi)^{\d}} \int_{\mathbb R^{\d}}\int_0^{2\pi}\hat \func(\freq)e^{-ib} 
    \right.\nonumber\\ & \qquad \left.
    \frac{1}{2}\left[e^{i(\freq^T \input + \b)} +  e^{-i(\freq^T \input + \b)}\right] d\b d\freq \right\}\nonumber\\
=  & \mathrm{Re}\left\{  \frac{1}{\pi} \frac{1}{(2\pi)^{\d}} \int_{\mathbb R^{\d}}\int_0^{2\pi}\hat \func(\freq) e^{-ib}\cos(\freq^T \input + \b) d\b d\freq \right\}\nonumber\\
\stackrel{\eqref{eq:c_as_real}}{=}& \frac{1}{(2\pi)^{\d}} \int_{\mathbb R^{\d}}\int_0^{2\pi}\amplTH(\freq,\b)\cos(\freq^T \input + \b)d\b d\freq.
\end{align}
\end{proof}

For $\b\in[0,2\pi]$, we have another useful expression for the ideal weights that is used later on in this section, namely
\begin{align}
\amplTH(\freq,\b) & = \frac{1}{\pi}\mathrm{Re}\left\{\hat \func(\freq) e^{-ib}\right\}\nonumber\\
& = \frac{1}{\pi}\mathrm{Re}\left\{\int_{\mathbb R^{\d}}\func(\input)e^{-i(\freq^T\input + b)}d\input\right\}\nonumber\\
& = \frac{1}{\pi}\int_{\mathbb R^{\d}}\func(\input)\cos(\freq^T\input + b)d\input. \label{eq:c_as_integral}
\end{align}

The function $\amplTH$ in Theorem~\ref{thm:zeroerror} is not unique.
However, of all functions $\ampl$ that satisfy \eqref{eq:zeroerror}, the given $\amplTH$ is the one with minimum norm.

\begin{theorem}
\label{thm:minnorm}
Let $\amplTH$ be as in Theorem~\ref{thm:zeroerror}. If $\tilde \ampl:\mathbb R^{\d}\times [0, 2\pi] \rightarrow \mathbb R$ satisfies
\begin{align}
	\func(\input) & = \frac{1}{(2\pi)^{\d}} \int_{\mathbb R^{\d}}\int_0^{2\pi}\tilde \ampl(\freq,\b)\cos(\freq^T \input + \b)d\b d\freq
    \quad \mathrm{a.e.}
    \label{eq:inv}
\end{align}   
then $||\tilde \ampl||^2_{L^2} \geq ||\amplTH||^2_{L^2} =\frac{(2 \pi)^\d }{\pi} ||\func||_{L^2}^2,$
with equality if and only if $\tilde \ampl = \amplTH$ in the $L^2$ sense.
\end{theorem}
 
\begin{proof}
First, using Parseval's theorem and $\int_0^{2\pi}\cos(a-b)^2 d\b = \pi$ for any real constant $a$, note that
\begin{align}
||\amplTH||_{L^2}^2 & = \int_{\mathbb R^{\d}}\int_0^{2\pi}\amplTH(\freq,\b)^2 d\b d\freq\nonumber\\
&\stackrel{\eqref{eq:c_as_cos}}{=} \int_{\mathbb R^{\d}}\int_0^{2\pi}\frac{1}{\pi^2}|\hat \func(\freq)|^2 \cos(\angle \hat\func(\freq) - \b)^2 d\b d\freq\nonumber\\
& = \int_{\mathbb R^{\d}}\frac{1}{\pi^2}|\hat \func(\freq)|^2 \int_0^{2\pi}\cos(\angle \hat \func(\freq) - \b)^2 d\b d\freq\nonumber\\
& = \int_{\mathbb R^{\d}}\frac{1}{\pi}|\hat \func(\freq)|^2 d\freq\nonumber\\
& = \frac{(2\pi)^{\d}}{\pi}\int_{\mathbb R^{\d}}\func(\input)^2 d\input =
\frac{(2\pi)^{\d}}{\pi}||\func||_{L^2}^2.\label{eq:normcisf}
\end{align}
Assume that $\tilde \ampl(\freq,\b) = \amplTH(\freq,\b) + \offampl(\freq,\b)$. Then we get
\begin{align}
	& \int_{\mathbb R^{\d}} \func(\input)^2 d\input\nonumber\\
    & \stackrel{\eqref{eq:inv}}{=} \int_{\mathbb R^{\d}} \func(\input) \frac{1}{(2\pi)^{\d}} \int_{\mathbb R^{\d}}\int_0^{2\pi} \tilde \ampl(\freq,\b)\cos(\freq^T \input + \b)d\b d\freq d\input \nonumber\\
	& = \frac{1}{(2\pi)^{\d}} \int_{\mathbb R^{\d}}\int_0^{2\pi} \tilde \ampl(\freq,\b) \int_{\mathbb R^{\d}} \func(\input) \cos(\freq^T \input + \b) d\input d\b d\freq \nonumber\\
    & \stackrel{\eqref{eq:c_as_integral}}{=} \frac{\pi}{(2\pi)^{\d}} \int_{\mathbb R^{\d}}\int_0^{2\pi} \tilde \ampl(\freq,\b) \amplTH(\freq,\b) d\b d\freq \nonumber\\
    & = \frac{\pi}{(2\pi)^{\d}} \int_{\mathbb R^{\d}}\int_0^{2\pi} \amplTH(\freq,\b)^2+ \amplTH(\freq,\b) \offampl(\freq,\b) d\b d\freq \nonumber\\
    & \stackrel{\eqref{eq:normcisf}}{=} \int_{\mathbb R^{\d}} \func(\input)^2 d\input + \frac{\pi}{(2\pi)^{\d}} \int_{\mathbb R^{\d}}\int_0^{2\pi}  \amplTH(\freq,\b) \offampl(\freq,\b) d\b d\freq.
\end{align}
Following the above equality we can conclude that
$\int_{\mathbb R^{\d}}\int_0^{2\pi}  \amplTH(\freq,\b) \offampl(\freq,\b) d\b d\freq = 0$. The following now holds:
\begin{align}
	||\tilde \ampl||^2_{L^2} & = ||\amplTH + \offampl||^2_{L^2} \nonumber\\
    & = \int_{\mathbb R^{\d}}\int_0^{2\pi} \amplTH(\freq,\b)^2+ 2\amplTH(\freq,\b) \offampl(\freq,\b) + \offampl(\freq,\b)^2 d\b d\freq \nonumber\\
	& = ||\amplTH||^2_{L^2} + ||\offampl||^2_{L^2} 
    \geq ||\amplTH||^2_{L^2} .
\end{align}

Furthermore, equality holds if and only if $||q||_{L^2}=0$. 
That is, the minimum norm solution is unique in $L^2$.

\end{proof}

These results will be used to derive ideal weights for a RFE with a finite number of basis functions as in Definition~\ref{def:RFE} by sampling the weights in~\eqref{eq:c_as_cos}. We prove unbiasedness in the following theorem, while variance properties are analyzed in Appendix~\ref{app:pw}.

\begin{theorem}\label{thm:unbiased}
For any continuous p.d.f. $p_{\freqRV}$ with $p_{\freqRV}(\freq)>0$ if $|\hat \func(\freq)|>0$, the choice
\begin{align}
\Ampl_{\k}  & = \frac{2}{\D (2\pi)^{\d}} \frac{|\hat \func(\freqRV_{\k})|}{p_{\freqRV}(\freqRV_{\k})}\cos(\angle \hat \func(\freqRV_{\k}) - \B_{\k})
\end{align}
makes the (stochastic) RFE $\RFERV(\input) = \sum_{\k=1}^\D \Ampl_{\k} \cos(\freqRV_{\k}^T \input + \B_{\k})$ an unbiased estimator, i.e., $f(\input)=\mathbb{E} [G(\input)]$ for any $\input \in \mathbb{R}^{\d}$.
\end{theorem}

\begin{proof}
Using Theorem~\ref{thm:zeroerror}, we have
\begin{align}
\func(\input) & = \frac{1}{(2\pi)^{\d}} \int_{\mathbb R^{\d}}\int_0^{2\pi}\amplTH(\freq,\b)\cos(\freq^T \input + \b)d\b d\freq \nonumber\\
 & = \mathbb E_{\freqRV_1,\B_1}\left[ \frac{1}{(2\pi)^{\d} p_{\B}(\B_1) p_{\freqRV}(\freqRV_1)} \amplTH(\freqRV_1,\B_1)\cos(\freqRV_1^T \input + \B_1) \right]\nonumber\\
 & =  \mathbb E_{\freqRV_{1\ldots \D},\B_{1\ldots \D}}\left[ \sum_{\k=1}^{\D} \frac{2\pi \amplTH(\freqRV_{\k},\B_{\k})}{\D(2\pi)^{\d} p_{\freqRV}(\freqRV_{\k})}  \cos(\freqRV_{\k}^T \input + \B_{\k})\right]\nonumber\\
 &\stackrel{\eqref{eq:c_as_cos}}{=} \mathbb E\left[ \sum_{\k=1}^{\D} \frac{2}{\D(2\pi)^{\d}}  \frac{|\hat \func(\freqRV_{\k})|}{p_{\freqRV}(\freqRV_{\k})}\cos(\angle \hat \func(\freqRV_{\k}) - \B_{\k})
 \right. \nonumber\\ & \left.  \qquad
\vphantom{\sum_{\k=1}^{\D}} \cos(\freqRV_{\k}^T \input + \B_{\k})\right]
= \mathbb E\left[\RFERV(\input)\right].
\end{align}
\end{proof}
These ideal weights enjoy many other nice properties such as infinity norm convergence~\cite{rahimi2008uniform}.
In practice, however, a least squares approach is used for a finite $\D$. 
This is investigated in the next subsection.

\subsection{Convergence of the Least Squares Solution}\label{sec:LSconv}

The ideal weights $\amplTH$ depend on the Fourier transform of the unknown function $\func$ that we wish to approximate. 
Of course, this knowledge is not available in practice.
We therefore assume a finite number of measurement points $\input_1, \ldots, \input_N$ that have been drawn independently from a p.d.f. $p_{\inputRV}$ that is defined on a compact set 
$\inputspace \subseteq \mathbb R^{\d}$, and corresponding measurements $y_1, \ldots, y_N$, with $y_n = \func(\input_n) + \eta_n$, where $\eta_1, \ldots, \eta_N$ have been drawn independently from a zero-mean normal distribution with finite variance $\sigma_\etaRV^2$. The input and noise terms are assumed independent of each other. We determine the weights $\ampl_\k$ by minimizing the squared error
\begin{align}\label{eq:LS}
J_{N}(\mathbf \ampl) & = \sum_{\iter=1}^N \left(y_{\iter} - \sum_{\k=1}^{\D} \ampl_{\k} \cos(\freq_{\k}^T \input_{\iter} + \b_{\k}) \right)^2
+  \reg\sum_{\k=1}^{\D}\ampl_{\k}^2\nonumber\\
 & = ||\mathbf y_N - \A_N \mathbf \ampl||_2^2 +  \reg ||\mathbf \ampl||_2^2.
\end{align}
Here, 
\begin{align}
\mathbf y_N & = \left[\begin{array}{c}y_1 \cdots y_N\end{array}\right]^T, \nonumber\\
\A_N & = \left[\begin{array}{ccc} \cos(\freq_{1}^T\input_1 + \b_1) & \cdots & \cos(\freq_{D}^T\input_1 + \b_D)\\ \vdots & \ddots & \vdots\\ \cos(\freq_{1}^T\input_N + \b_1) & \cdots & \cos(\freq_{D}^T\input_N + \b_D)\end{array}\right],\label{eq:yandAdef}
\end{align}
and $\reg$ is a regularization parameter added to deal with noise, over-fitting and ill-conditioning.

Since the parameters $\freq_{\k}, \b_{\k}$ are drawn from continuous probability distributions, only the weights $\ampl_\k$ need to be determined, making the problem a linear least squares problem. 
The unique minimizer of $J_N$ is
\begin{align} \label{eq:LSclosed}
    \mathbf \amplLS
    & = \left(\A_N^T \A_N + \reg \mathbf{I}_{\D\times\D}\right)^{-1}\A_N^T \mathbf y_N.
\end{align}

The following theorem shows that RFEs whose coefficient vector have been obtained through a least squares fit as in \eqref{eq:LSclosed} can approximate the function $\func$ arbitrarily well.
Similar results were given in~\cite{girosi1992convergence, barron1993universal, rahimi2008uniform, jones1992simple}, but we emphasize that these convergence results did concern RFEs employing the ideal coefficient vector given earlier in Theorem~\ref{thm:unbiased} that is unknown in practice. 
Our theorem, in contrast, concerns the practically relevant case where the coefficient vector has been obtained through a least-squares fit to the data.
\begin{theorem}\label{thm:LSconv}
	The difference between the function $\func$ and the RFE trained with linear least squares can become arbitrarily small if enough measurements and basis functions are used.
    More precisely, suppose that $\func \in L^2 \cap L^{\infty}$
     and that 
$\sup_{\freq \in \mathbb R^{\D},\b\in[0,2\pi]} \left|\frac{\amplTH(\freq,\b)}{p_{\freqRV}(\freq) p_{\B}(\b)}\right| < \infty$.
     Then, for every $\epsilon>0$ and $\delta>0$, there exist constants $N_0$ and $\D_0$ such that
     \begin{align}
     	\int_{\inputspace} \left(\func(\input) -  \sum_{\k=1}^{\D} \amplLSRV_{\k} \cos(\freqRV_{\k}^T \input + \B_{\k}) \right)^2 p_{\inputRV}(\input) d\input < \epsilon
	\end{align}
for all $N\geq N_0$, $\D\geq \D_0$, $0<\reg \leq N\Lambda$      
with probability at least $1-\delta$. 
Here, 
$\amplLSRV_{\k}$ is the $\k$-th element of the random vector corresponding to the weight vector given in~\eqref{eq:LSclosed},
and $\Lambda\ge 0$ is the solution to
\begin{align}\label{eq:LambdaEQ}
& \left|\left|	\left(  \A_N^T \A_N + N \Lambda \  \mathbf{I}_{\D\times\D}\right)^{-1}\A_N^T \mathbf y_N\right|\right|_2^2 = 
\nonumber\\
 & \sum_{\k=1}^{\D}\left(\frac{\amplTH(\freq_{\k},\b_{\k})}{(2\pi)^{\d}\D p_{\freqRV}(\freq_{\k})p_{\B}(\b_{\k})}\right)^2.
\end{align}
\end{theorem}

The proof of this theorem is given in Appendix~\ref{ap:proofLStheorem}. In Section~\ref{sec:Ruleofthumbreg} we show how to obtain $\Lambda$ in practice.

\section{Online Optimization Algorithm}\label{sec:algstep1}
In this section, we will investigate the DONE algorithm, which locates a minimum of an unknown function $f$ based on noisy evaluations of this function. 
Each evaluation, or \emph{measurement}, is used to update a RFE model of the unknown function, based on which the next measurement point is determined.
Updating this model has a constant computation time of order $O(\D^2)$ per iteration, with $\D$ being the number of basis functions.
We emphasize that this is in stark contrast to Bayesian optimization algorithms, where the computational cost of adding a new measurement increases with the total number of measurements so far.
We also remark that the DONE algorithm operates \emph{online} because the model is updated after each measurement. The advantage over offline methods, in which first all measurements are taken and only then processed, is that the number of required measurements is usually lower as measurement points are chosen adaptively.

\subsection{Recursive Least Squares Approach for the Weights}\label{sec:RLSexplain}
In the online scenario, a new measurement $y_{\iter}$ taken at the point $\input_{\iter}$ becomes available at each iteration $n=1,2, \dots$ These are used to update the RFE. Let $\a_{\iter} = [\cos(\freq_1^T \input_{\iter} + \b_1) \cdots \cos(\freq_{\D}^T \input_{\iter} + \b_{\D})]$, then we aim to find the vector of RFE weights by minimizing the regularized mean square error
\begin{align}\label{eq:objLS2}
J_{{\iter}}(\mathbf \ampl) & = \sum_{i=1}^{\iter} \left(y_i - \a_i \mathbf \ampl \right)^2
+ \reg ||\mathbf \ampl||_2^2.
\end{align}
Let $\mathbf \ampl_{\iter}$ be the minimum of $J_n$,
\begin{align}\label{eq:minJ}
\mathbf \ampl_{\iter} = \argmin_{\mathbf \ampl} J_{\iter}(\mathbf \ampl).
\end{align}
Assuming we have found $\mathbf \ampl_{\iter}$, we would like to use this information to find $\mathbf \ampl_{{\iter}+1}$ without solving~\eqref{eq:minJ} again.
The recursive least squares algorithm
is a computationally efficient method that determines $\mathbf\ampl_{\iter+1}$ from $\mathbf\ampl_{\iter}$ as follows~\cite[Sec. 21]{sayed1998recursive}:
\begin{align}
\gamma_{\iter} & = 1/(1+\a_{\iter} \mathbf{P}_{{\iter}-1} \a_{\iter}^T),\label{eq:rulegamma}\\
\mathbf g_{\iter} & = \gamma_{\iter} \mathbf{P}_{{\iter}-1} \a_{\iter}^T,\label{eq:ruleg}\\
\mathbf \ampl_{\iter} & = \mathbf \ampl_{{\iter}-1} + \mathbf g_{\iter} (y_{\iter} - \a_{\iter} \mathbf \ampl_{{\iter}-1}),\label{eq:rulec}\\
\mathbf{P}_{\iter} & = \mathbf{P}_{{\iter}-1} - \mathbf g_{\iter} \mathbf g_{\iter}^T/\gamma_{\iter},\label{eq:ruleP}
\end{align}
with initialization $\mathbf \ampl_0 = 0$, $\mathbf{P}_0 = \reg^{-1} \mathbf I_{\D \times \D}$. 

We implemented a square-root version of the above algorithm, also known as the inverse QR algorithm~\cite[Sec. 21]{sayed1998recursive}, which is known to be especially numerically reliable. Instead of performing the update rules~\eqref{eq:rulegamma}-\eqref{eq:ruleP} explicitly, we find a rotation matrix $\mathbf \Theta_{\iter}$ that lower triangularizes the upper triangular matrix in Eq.~\eqref{eq:QRtriang} below and generates a post-array with positive diagonal entries:
\begin{align}
\label{eq:QRtriang}
\left[\begin{array}{cc} 1 & \a_{\iter} \mathbf{P}^{1/2}_{{\iter}-1}\\ 
\mathbf 0 & \mathbf{P}^{1/2}_{{\iter}-1}\end{array}\right]
\mathbf{\Theta}_{\iter} =
\left[\begin{array}{cc} \gamma^{-1/2}_{\iter} & \mathbf 0\\ 
\mathbf{g}_{\iter} \gamma^{-1/2}_{\iter} & \mathbf{P}^{1/2}_{\iter} \end{array}\right].
\end{align} 
The rotation matrix $\mathbf \Theta_{\iter}$ can be found by performing a QR decomposition of the transpose of the matrix on the left hand side of~\eqref{eq:QRtriang}, or by the procedure explained in~\cite[Sec. 21]{sayed1998recursive}. 
The computational complexity of this update is $O(\D^2)$ per iteration.

\subsection{DONE Algorithm}\label{sec:DONEexplanation}
We now explain the different steps of the DONE algorithm.
The DONE algorithm is used to iteratively find a minimum of a function $\func \in L^2
$ on a compact set $\inputspace \subseteq \mathbb R^{\d}$ by updating a RFE
$\RFE(\input) = \sum_{\k=1}^\D \ampl_{\k} \cos(\freq_{\k}^T \input + \b_{\k})$
at each new measurement, and using this RFE as a surrogate of $\func$ for optimization.
It is assumed that the function $\func$ is unknown and only measurements perturbed by noise can be obtained: $y_{\iter} = \func(\input_{\iter}) + \eta_{\iter}$.
The algorithm consists of four steps that are repeated for each new measurement: \textbf{1)} take a new measurement, \textbf{2)}  update the RFE, \textbf{3)} find a minimum of the RFE, \textbf{4)} choose a new measurement point. 
We now explain each step in more detail.

\bigskip

\textbf{Initialization}

Before running the algorithm, an initial starting point $\input_1 \in \inputspace$ and the number of basis functions $\D$ have to be chosen.
The parameters $\freq_\k$ and $\b_\k$ of the RFE expansion are drawn from continuous probability distributions as defined in Definition~\ref{def:RFE}.
The p.d.f. $p_\freqRV$ and the regularization parameter $\reg$ have to be chosen a priori as well.
Practical ways for choosing the hyper-parameters will be discussed later in Sect.~\ref{sec:hyp}.
These hyper-parameters stay fixed over the whole duration of the algorithm.
Let $\mathbf{P}^{1/2}_0 = \reg^{-1/2} \mathbf I_{D \times D}$, and $\iter = 1$.

\bigskip
\textbf{Step 1: New measurement}

Unlike in Section~\ref{sec:LSconv}, it is assumed 
that measurements are taken in a recursive fashion.
At the start of iteration $n$, a new measurement $y_{\iter} = \func(\input_{\iter})+\eta_{\iter}$ is taken at the point $\input_{\iter}$.

\bigskip
\textbf{Step 2: Update the RFE}

As explained in Section~\ref{sec:RLSexplain}, we update the RFE model $\RFE(\input) = \sum_{\k=1}^\D \ampl_{\k} \cos(\freq_{\k}^T \input + \b_{\k})$ based on the new measurement from Step 1 by using the inverse QR algorithm given in ~\eqref{eq:rulegamma}-\eqref{eq:ruleP}.
Only the weights $\ampl_{k}$ are updated. 
The parameters $\freq_{\k}$ and $\b_{\k}$ stay fixed through-out the whole algorithm.

\bigskip
\textbf{Step 3: Optimization on the RFE}\label{sec:algstep2}

After updating the RFE, 
an iterative optimization algorithm is used to find a (possibly local) minimum $\hat{\input}_{\iter}$ of the RFE.
All derivatives of the RFE can easily be calculated. 
Using an analytic expression of the Jacobian will increase the performance of the optimization method used in this step, while not requiring extra measurements of $\func$ as in the finite difference method. 
For functions that are costly to evaluate, this is a big advantage.
The method used in the proposed algorithm is an L-BFGS method~\cite{nocedal1980updating, nocedal2006numerical}.
Other optimization methods can also be used.
%
The initial guess for the optimization 
is the projection of the current measurement point plus a random perturbation:
\begin{align}
	\input_{init} = P_{\mathcal X}(\input_{\iter} + {\perturb}_{\iter}),
\end{align}
where $P_{\mathcal X}$ is the projection onto $\mathcal X$.
The random perturbation prevents the optimization algorithm from starting exactly in the point where the model was trained. Increasing its value will increase the exploration capabilities of the DONE algorithm but might slow down convergence.
In the proposed algorithm, ${\perturb_{\iter}}$ is chosen to be white Gaussian noise. 

\bigskip
\textbf{Step 4: Choose a new measurement point}\label{sec:algstep3}

The minimum found in the previous step is used to update the RFE again. 
A perturbation is added to the current minimum
to avoid the algorithm getting trapped unnecessarily in insignificant local minima or saddle points~\cite{pogu1994global}:
\begin{align}
	\input_{\iter+1} = P_{\mathcal X}(\hat{\input}_{\iter} + \xi_{\iter}).
\end{align}
The random perturbations can be seen as an exploration strategy and are again chosen to be white Gaussian noise. 
Increasing their variance $\sigma_{\xi}$ increases the exploration capabilities of the DONE algorithm but might slow down convergence. In practice, we typically
use the same distribution for $\xi$ and $\perturb$.
Finally, the algorithm increases $\iter$ and returns to Step $1$. 

\bigskip

The full algorithm is shown below in Algorithm~\ref{alg:DONE} for the case $\inputspace = [lb, ub]^{\d}$.
\begin{algorithm}
\caption{DONE Algorithm}\label{alg:DONE}
\begin{algorithmic}[1]
\Procedure{DONE}{$\func, \input_1, N, lb, ub, \D, \reg, \sigma_{\perturb},\sigma_{\xi}$}
\State Draw $\freq_1 \ldots \freq_\D$ from $p_{\freqRV}$ independently.
\State Draw $\b_1 \ldots \b_\D $ from $\operatorname{Uniform}(0,2\pi)$ independently.
\State $\mathbf{P}^{1/2}_0 = \reg^{-1/2} \mathbf I_{D \times D}$
\State $\mathbf \ampl_0 = [0 \ldots 0]^T$
\State $\hat{\input}_0 = \input_1$
\For{$\iter = 1, 2, 3, \ldots, N$ 
}
\State ${\a}_{\iter} = [\cos(\freq_1^T \input_{\iter} + \b_1) \cdots \cos(\freq_{\D}^T \input_{\iter} + \b_{\D})]$
\State $y_{\iter} = \func(\input_{\iter}) + \eta_{\iter}$
\State $\RFE(\input) = \operatorname{updateRFE}(\mathbf \ampl_{\iter-1}, \mathbf P_{\iter-1}^{1/2}, {\a}_{\iter},  y_{\iter})$
\State Draw $\perturb_{\iter}$ from $\mathcal N(0,\sigma^2_{\perturb} \mathbf I_{\d \times \d})$.
\State $\input_{init} = \max(\min(\input_{\iter}+{\perturb}_{\iter}, ub), lb)$
\State $[\hat{\input}_{\iter},\hat{\RFE}_{\iter}] = \operatorname{L-BFGS}(\RFE(\input), \input_{init}, lb, ub)$
\State Draw $\xi_{\iter}$ from $\mathcal N(0,\sigma^2_{\xi} \mathbf I_{\d \times \d})$.
\State $\input_{\iter+1} = \max(\min(\hat{\input}_{\iter}+\xi_{\iter},ub),lb)$ 
\EndFor\label{DoneendFor}
\State \textbf{return} $\hat{\input}_{\iter}$
\EndProcedure
\end{algorithmic}
\end{algorithm}

\begin{algorithm}
\caption{updateRFE}\label{alg:FR}
\begin{algorithmic}[1]
\Procedure{updateRFE}{$\mathbf \ampl_{\iter-1}, \mathbf P_{\iter-1}^{1/2}, {\a}_{\iter},  y_{\iter}$}
\State Retrieve $\mathbf g_{\iter} \gamma_{\iter}^{-1/2}$, $\gamma_{\iter}^{-1/2}$ and $\mathbf{P}^{1/2}_{\iter}$ from~\eqref{eq:QRtriang}
\State $\mathbf\ampl_{\iter} = \mathbf\ampl_{\iter-1} + {\mathbf{g}_{\iter}} ( y_{\iter} - {\a}_{\iter} \mathbf\ampl_{\iter-1})$
\State $\RFE(\input) = [\cos(\freq_1^T \input + \b_1) \cdots \cos(\freq_{\D}^T \input + \b_{\D})]\mathbf \ampl_{\iter}$
\State \textbf{return} $\RFE(\input)$
\EndProcedure
\end{algorithmic}
\end{algorithm}

\section{Choice of Hyper-parameters}\label{sec:hyp}

In this section, we will analyze the influence of the hyper-parameters of the DONE algorithm and, based on these results, provide practical ways of choosing them. 
The performance of DONE depends on the following hyper-parameters:
\begin{itemize}
	\item number of basis functions $\D$,
    \item p.d.f. $p_{\freqRV}$,
    \item regularization parameter $\reg$,
    \item exploration parameters $\sigma_{\perturb}$ and $\sigma_{\xi}$.
\end{itemize}

The influence of $\D$ is straight-forward: increasing $\D$ will lead to a better performance (a better RFE fit) of the DONE algorithm at the cost of more computation time. 
Hence, $D$ should be chosen high enough to get a good approximation, but not too high to avoid unnecessarily high computation times. It should be noted that $D$ does not need to be very precise. 
Over-fitting should not be a concern for this parameter since we make use of regularization.
The exploration parameters determine the trade-off between exploration and exploitation, similar to the use of the acquisition function in Bayesian optimization~\cite{brochu2010tutorial, snoek2012practical}. The parameter $\sigma_{\perturb}$ influences the exploration of the RFE surrogate in Step 3 of the DONE algorithm, while $\sigma_{\xi}$ determines exploration of the original function. Assuming both to be close to each other, $\sigma_{\perturb}$ and $\sigma_{\xi}$  are usually chosen to be equal. If information about local optima of the RFE surrogate or of the original function is available, this could be used to determine good values for these hyper-parameters. Alternatively, similar to Bayesian optimization the expected improvement could be used for that purpose, but this remains for future work.
The focus of this section will be on choosing $p_{\freqRV}$ and $\reg$.

\subsection{Probability Distribution of Frequencies}\label{sec:hypfreq}

Recall the parameters $\freq_{\k}$ and $\b_{\k}$  from Definition~\ref{def:RFE}, which are obtained by sampling independently from the continuous probability distributions $p_{\freqRV}$ and $p_{\B} = \mathrm{Uniform(0,2\pi)}$, respectively.
In the following, we will investigate the first and second order moments of the RFE and try to find a distribution $p_{\freqRV}$ that minimizes the variance of the RFE.
 
Unfortunately, as shown in Theorem~\ref{thm:minvar} in Appendix~\ref{app:pw}, it turns out that the optimal p.d.f. is
\begin{align}\label{eq:optpdfB}
    p_{\freqRV}^*(\freq) & = \frac{|\hat \func(\freq)|\sqrt{\cos(2\angle \hat\func(\freq)+2\freq^T\input) + 2}}{\int_{\mathbb R^{\d}} |\hat \func(\tilde{\freq})| \sqrt{\cos(2\angle \hat\func(\tilde{\freq})+2\tilde{\freq}^T\input) + 2} d\tilde{\freq}}.
\end{align}
This distribution depends on the input $\input$ and both the phase and magnitude
of the Fourier transform of $\func$. 
But if both $|\hat \func|$ and $\angle \hat \func$ were known, then the function $\func$ itself would be known, and standard optimization algorithms could be used directly. 
Furthermore, we would like to use a p.d.f. for $\freq_{\k}$ that does not depend on the input $\input$, since the $\freq_{\k}$ parameters are chosen independently from the input in the initialization step of the algorithm.

In calibration problems, the objective function $\func$ suffers from an unknown offset, $\func(\input) = \tilde \func(\input + \Delta)$.
This unknown offset does not change the magnitude in the Fourier domain, but it does change the phase.
Since the phase is thus unknown, we choose a uniform distribution for $p_{\B}$ such that $\b_{\k}\in[0,2\pi]$.
However, the magnitude $|\hat \func|$ can be measured in this case. Section~\ref{sec:OCT} describes an example of such a problem. 
We will now derive a way to choose $p_{\freqRV}$ for calibration problems.

In order to get a close to optimal p.d.f. for $\freq_{\k}$ that is independent of the input $\input$ and of the phase $\angle \hat \func$ of the Fourier transform of $\func$, we look at a complex generalization of the RFE. In this complex problem, it turns out we can circumvent the disadvantages mentioned above by using a p.d.f. that depends only on $|\hat \func|$.

\begin{theorem}\label{thm:minvarcomplex}
Let $\tilde \RFERV(\input) = \sum_{\k=1}^\D \tilde{\Ampl}_{\k} e^{i \freqRV_{\k}^T \input + \B_{\k}}$, with $\freqRV_{\k}$ being i.i.d. random vectors with a continuous p.d.f. $\tilde p_{\freqRV}$ over $\mathbb{R}^d$
that satisfies $\tilde p_{\freqRV_{\k}}(\freq)>0$ if $|\hat \func(\freq)| > 0$, and $\B_{\k}$ being random variables with uniform distribution from $[0, 2\pi]$.
Then $\tilde \RFERV(\input)$ is an unbiased estimator of $\func(\input)$ for all $\input \in \mathbb{R}^{\d}$ if 
\begin{align}\tilde \Ampl_{\k} = \frac{\hat \func(\freqRV_{\k}) e^{-i\B_{\k}}}{D(2\pi)^{\d}\tilde p_{\freqRV}(\freqRV_{\k})}.\end{align}
For this choice of $\tilde \Ampl_{\k}$, the variance of $\tilde \RFERV(\input)$ is minimal if \begin{align}\tilde p_{\freqRV}(\freq) = \frac{|\hat \func(\freq)|}{\int_{\mathbb R^{\d}} |\hat \func(\tilde \freq)|d\tilde \freq},\label{eq:optpdfwComplex}\end{align}
giving a variance of
\begin{align}\label{eq:minvarcomplex}
\mathrm{Var}[\tilde \RFERV(\input)]
    & = \frac{1}{\D (2\pi)^{2\d}} \left(\int_{\mathbb R^{\d}} |\hat \func(\freq)| d\freq \right)^2- \func(\input)^2.\nonumber\\
\end{align}
\end{theorem}
\begin{proof}
The unbiasedness follows directly from the  Fourier inversion theorem,
\begin{align}
	\mathbb E\left[ \tilde \RFERV(\input)\right] & =  \sum_{\k=1}^\D \int_{\mathbb R^{\d}} \int_0^{2\pi} \frac{\hat \func(\freq_{\k})e^{-i\b_{\k}}e^{i \freq_{\k}^T \input + \b_{\k}}}{D(2\pi)^{\d}\tilde p_{\freqRV}(\freq_{\k})2\pi}d\b_{\k}\tilde p_{\freqRV}(\freq_{\k}) d\freq_{\k} \nonumber\\
    & = \D \int_{\mathbb R^{\d}} \int_0^{2\pi}\frac{\hat \func(\freq)e^{-i\b}}{D(2\pi)^{\d}\tilde p_{\freqRV}(\freq)} e^{i \freq^T \input + \b} \frac{1}{2\pi}d\b\tilde p_{\freqRV}(\freq) d\freq\nonumber\\
    & = \D \int_{\mathbb R^{\d}} \frac{\hat \func(\freq)}{D(2\pi)^{\d}\tilde p_{\freqRV}(\freq)} e^{i \freq^T \input} \tilde p_{\freqRV}(\freq)
    \int_0^{2\pi}
    \frac{1}{2\pi}d\b
    d\freq    \nonumber\\
    & = \frac{1}{(2\pi)^{\d}}\int_{\mathbb R^{\d}}\hat \func(\freq) e^{i \freq^T \input}  d\freq\nonumber\\
    & = \func(\input).
\end{align}  
 The proof of minimum variance is similar to the proof of~\cite[Thm. 4.3.1]{rubinstein2011simulation}.
\end{proof}

Note that the coefficients $\tilde \Ampl_{\k}$ can be complex in this case.
Next, we show that the optimal p.d.f. for a complex RFE, $\tilde p_{\freqRV}$, is still close-to-optimal (in terms of the second moment) when used in the real RFE from Definition~\ref{def:RFE}. 
\begin{theorem}\label{thm:varbound}
Let $\tilde p_{\freqRV}$ be as in~\eqref{eq:optpdfwComplex}
and let $\RFERV$ with weights $\Ampl_{\k}$ be as in Theorem~\ref{thm:unbiased}. 
Let $P$ be the set of probability distribution functions for $\freqRV_{\k}$ that are positive when $|\hat\func(\freq)|>0$.
Then, we have
\begin{align}
\mathbb{E}_{\tilde p_{\freqRV},p_{\B}}[\RFERV(\input)^2] \leq \sqrt 3 \ \min_{p_{\freqRV} \in P} \mathbb{E}_{p_{\freqRV},p_{\B}}[\RFERV(\input)^2].\label{eq:minvarcomplexAA}
\end{align}
\end{theorem}
The proof is given in Appendix~\ref{app:pw}. We now discuss how to choose $p_{\freqRV}$ in practice.

If no information of $|\hat \func|$ is available, the standard approach of choosing $p_{\freqRV}$ as a zero-mean normal distribution can be used. The variance $\sigma^2$ is an important hyper-parameter in this case, and any method of hyper-parameter tuning can be used to find it. However, most hyper-parameter optimization methods are computationally expensive because they require running the whole algorithm multiple times. In the case that $|\hat \func|$ is not exactly known, but some information about it is available (because it can be estimated or measured for example), this can be circumvented. The variance $\sigma^2$ can simply be chosen in such a way that $p_{\freqRV}$ most resembles the estimate for $|\hat \func|$, using standard optimization techniques or by doing this by hand. In this approach, it is not necessary to run the algorithm at all, which is a big advantage compared to most hyper-parameter tuning methods.
All of this leads to a rule of thumb for choosing $p_{\freqRV}$ as given in  Algorithm~\ref{alg:pwROT}.

\begin{algorithm}
\caption{Rule of thumb for choosing $p_{\freq}$}\label{alg:pwROT}
\begin{algorithmic}[1]
\If{$|\hat \func|$ is known exactly}
\State Set $p_{\freqRV} = |\hat \func|/
\int |\hat \func(\freq)| d\freq$.
\Else
\State Measure or estimate  $|\hat \func|$.
\State Determine $\sigma^2$ for which the pdf of $\mathcal N(0,\sigma^2 \mathbf{I}_{\d \times \d})$ is close in shape to $|\hat \func|/\int |\hat \func(\freq)| d\freq$.
\State Set $p_{\freqRV} = \mathcal N(0,\sigma^2 \mathbf{I}_{\d \times \d})$.
\EndIf
\end{algorithmic}
\end{algorithm}

\subsection{Upper Bound on the Regularization Parameter}
\label{sec:Ruleofthumbreg}

The regularization parameter $\reg$ in the performance criterion ~\eqref{eq:LS} is used to prevent under- or over-fitting of the RFE under noisy conditions or when dealing with few measurements.
Theorem~\ref{thm:LSconv} guarantees the convergence of the least squares solution only if the regularization parameter satisfies $\reg\leq N\Lambda$, where $N$ is the total number of samples and $\Lambda$ is defined in \eqref{eq:LambdaEQ}.
Here we will provide a method to estimate $\Lambda$.
 
During the proof of Theorem~\ref{thm:LSconv}, it was shown that the upper bound $\Lambda$ corresponds to the $\lambda$ that satisfies
\begin{align}
& \left|\left|	\left( \A_N^T \A_N +  N\lambda \  \mathbf{I}_{\D\times\D}\right)^{-1} \A_N^T \mathbf y_N\right|\right|_2^2 \nonumber\\
& = \sum_{\k=1}^{\D}\left(\frac{\amplTH(\freq_{\k},\b_{\k})}{(2\pi)^{\d}\D p_{\freqRV}(\freq_{\k})p_{\B}(\b_{\k})}\right)^2 = M^2. \label{eq:upperboundLA}
\end{align}
The left-hand side in this equation is easily evaluated for different values of $\lambda$. 
Thus, in order to estimate $\Lambda$, all we need is an approximation of the unknown right hand $M^2$.

Like in Section~\ref{sec:hypfreq}, it is assumed that no information about $\angle \hat \func$ is available, but that $|\hat \func|$ can be measured or estimated. 
Under the assumptions that $\D$ is large and that $p_{\freqRV}$ is a good approximation of $\tilde p_{\freqRV} = |\hat \func(\freq)|/ \int_{\mathbb R^{\d}} |\hat \func(\freq)| d\freq$ as in Algorithm~\ref{alg:pwROT}, we obtain the following approximation of $M$:
\begin{align}
M & = \frac{2}{(2\pi)^{\d}} \sqrt{\frac{1}{\D^2}\sum_{\k=1}^{\D}\left(\frac{|\hat \func(\freq_{\k})|}{ p_{\freqRV}(\freq_{\k})} \cos(\angle \hat \func(\freq_{\k}) - \b_{\k})\right)^2}\nonumber\\
 & \approx \frac{2}{(2\pi)^{\d}} \sqrt{\frac{1}{\D}\mathbb E\left[\left(\frac{|\hat \func(\freqRV_1)|}{ p_{\freqRV}(\freqRV_1)} \cos(\angle \hat \func(\freqRV_1) - \B_1)\right)^2\right]}\nonumber\\
 & = \frac{2}{(2\pi)^{\d}} \sqrt{\frac{1}{2\pi\D}\int_{\mathbb R^{\d}}\int_0^{2\pi}\frac{|\hat \func(\freq)|^2}{ p_{\freqRV}(\freq)} \cos^2(\angle \hat \func(\freq) - \b) d\b d\freq}\nonumber\\
 & = \frac{\sqrt{2}}{(2\pi)^{\d} \sqrt{\D}} \sqrt{\int_{\mathbb R^{\d}}\frac{|\hat \func(\freq)|^2}{ p_{\freqRV}(\freq)} d\freq}\nonumber\\
 & \approx \frac{\sqrt{2}}{(2\pi)^{\d} \sqrt{\D}} \sqrt{\int_{\mathbb R^{\d}}\frac{|\hat \func(\freq)|^2}{ \tilde p_{\freqRV}(\freq)} d\freq}\nonumber\\
 & = \frac{\sqrt{2}}{(2\pi)^{\d} \sqrt{\D}}\int |\hat \func(\freq)| d\freq = M_a.
\end{align}
The squared cosine was removed as in Eq.~\eqref{eq:normcisf}.
Using the exact value or an estimate of $\int_{\mathbb R^{\d}} |\hat \func(\freq)| d\freq$ as in Algorithm~\ref{alg:pwROT} to determine $M_a$, we 
calculate the left-hand in~\eqref{eq:upperboundLA}
for multiple values of $\Lambda$ and take the value for which 
it is closest to $M_a^2$.
The procedure is summarized in Algorithm~\ref{alg:lambdaROT}.

\begin{algorithm}
\caption{Rule of thumb for finding an estimate of $\Lambda$}\label{alg:lambdaROT}
\begin{algorithmic}[1]
\State Run Algorithm~\ref{alg:pwROT} to get $\int_{\mathbb R^{\d}} |\hat \func(\freq)| d\freq$.
\State Take $N$ measurements to get $\A_N$ and $\mathbf y_N$.
\State Determine $\Lambda$ for which 
the left-hand side of~\eqref{eq:upperboundLA}
is close to 
$M_a^2 = \frac{2}{(2\pi)^{2\d} {\D}}\left(\int |\hat \func(\freq)| d\freq\right)^2$.
\end{algorithmic}
\end{algorithm}

\section{Numerical Examples}\label{sec:appl}

In this section, we compare the DONE algorithm to the Bayesian optimization library  BayesOpt~\cite{martinez2014} in several numerical examples.

\subsection{Analytic Benchmark Problem: Camelback Function}\label{sec:camelback}
The camelback function 
\begin{align}
f(\input) = \left(4 - 2.1x_1^2 + \frac{x_1^4}{3}\right) x_1^2 + x_1 x_2 + \left(-4 + 4x_2^2\right) x_2^2,
\end{align}
where $\input = [x_1,x_2]\in [-2,2]\times [-1, 1]$, is a standard test function with two global minima and two local minima. 
The locations of the global minima are approximately $(0.0898, -0.7126)$ and $(-0.0898, 0.7126)$ with an approximate function value of $-1.0316$.
We determined the hyper-parameters for DONE on this test function as follows. 
First, we computed the Fourier transform of the function. 
We then fitted a function $h(\freq) = \frac{C}{\sigma\sqrt{2\pi}} e^{-\frac{\freq^2}{2\sigma^2}}$ to the magnitude of the Fourier transform in both directions. This was done by trial and error, giving a value of $\sigma = 10$.
To validate, two RFEs were fit to the original function using a normal distribution with standard deviation $\sigma=10$ (good fit) and $\sigma=0.1$ (bad fit) for $\freq_{\k}$, using the least squares approach from Section~\ref{sec:LSconv}. 
Here, we used $N=1000$ measurements sampled uniformly from the input domain, the number of basis functions $\D$ was set to $500$, and a regularization parameter of $\reg = 10^{-10}$ was used. 
The small value for $\reg$ still works well in practice because the function $f$ does not contain noise.

Let $\RFE(\input)$ denote the value of the trained RFE at point $\input$. We investigated the root mean squared error (RMSE)
\begin{align}
	\mathrm{RMSE} = \sqrt{\frac{1}{N} \sum_{n=1}^N (f(\input_n)-g(\input_n))^2},
\end{align}
for the two stated values of $\sigma$.
The good fit gave a RMSE of $5.5348\cdot 10^{-6}$, while the bad fit gave a RMSE  of $0.2321$, which shows the big impact of this hyper-parameter on the least squares fit.

We also looked at the difference between using the real RFE from Definition~\ref{def:RFE} and the complex RFE from~Theorem~\ref{thm:minvarcomplex}, for $\sigma=10$,
and for different values of $\D$ ($\D \in \{10, 20, 40, 80, 160, 320, 640, 1280\}$).
Fig.~\ref{fig:realvscomplex} shows the mean and standard deviation of the RMSE 
over $100$ runs. 
We see that the real RFE 
indeed performs similar to the complex RFE as predicted by Theorem~\ref{thm:Vardistances} in Appendix~\ref{app:pw}.
\begin{figure}[htbp]
\centerline{
\includegraphics[width=0.9\columnwidth]
{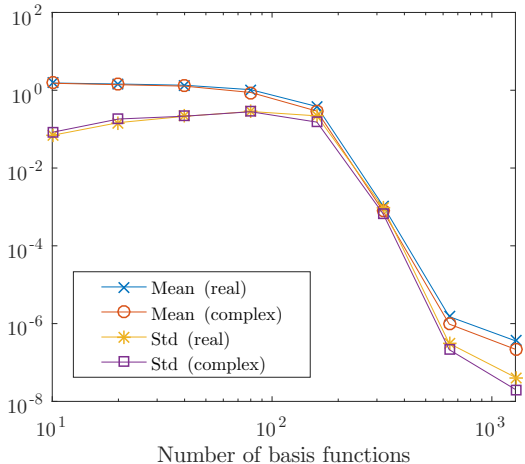}}
\caption{Mean and standard deviation of the root mean square error for a real and a complex RFE over $100$ runs. 
}
\label{fig:realvscomplex}
\end{figure}

Using the hyper-parameters $\sigma =10$ and $\reg = 10^{-10}$, we also performed $10$ runs of the DONE algorithm and compared it to reproduced results from~\cite[Table 1]{martinez2014} (method ``BayesOpt1''). 
The number of basis functions $\D$ was set to $500$, one of the smallest values with a RMSE of below $10^{-5}$ according to Fig.~\ref{fig:realvscomplex}, and the initial guess was chosen randomly. The exploration parameters $\sigma_{\perturb}$ and $\sigma_{\xi}$ were set to $0.01$. The resulting distance to the true minimum and the computation time in seconds (with their standard deviations) for $50$ and $100$ measurements can be found in Table~\ref{tab:benchmarkresults}. As in~\cite{martinez2014}, the computation time for BayesOpt was only shown for $100$ samples and the accuracy below $10^{-5}$ was not shown.
It can be seen that the DONE algorithm is several orders of magnitude more accurate and about $5$ times faster when compared to BayesOpt for this problem.

\begin{table}[htbp]
\renewcommand{\arraystretch}{1.3}
\caption{DONE vs BayesOpt on the Camelback function}
\label{tab:benchmarkresults}
\centering
\begin{tabular}{|c|c|c|}
\hline
 & Dist. to min. ($50$ samp.) & Time ($50$ samp.) \\
\hline
DONE & $2.1812\cdot 10^{-9} \ (8.3882\cdot 10^{-9})$ & $0.0493\  (0.0015)$ \\
\hline
BayesOpt & $0.0021 \ (0.0044)$ & - \\
\hline
\hline
 & Dist. to min. ($100$ samp.) & Time ($100$ samp.) \\
\hline
DONE & $  1.1980 \cdot 10^{-9} \ (5.2133 \cdot 10^{-9})$ & $0.0683\  (0.0019)$\\
\hline
BayesOpt & $<1\cdot 10^{-5} \ (<1\cdot 10^{-5})$ & $0.3049 \ (0.0563)$ \\
\hline
\end{tabular}
\end{table}

\subsection{Optical Coherence Tomography}\label{sec:OCT}

Optical coherence tomography (OCT) is a low-coherence interferometry imaging technique used for making three-dimensional images of a sample. 
The quality and resolution of images is reduced by optical wavefront aberrations caused by the medium, e.g., the human cornea when imaging the retina.
These aberrations can be removed by using active components such as deformable mirrors in combination with optimization algorithms \cite{Bonora13,Verstraete15}.
The arguments of the optimization can be the voltages of the deformable mirror or a mapping of these voltages to other coefficients such as the coefficients of Zernike polynomials.
The intensity of the image at a certain depth is then maximized to remove as much of the aberrations as possible.
In \cite{Verstraete15} it was shown experimentally that the DONE algorithm greatly outperforms other derivative-free algorithms in final root mean square (RMS) wavefront error and image quality.
Here, we numerically compare the DONE algorithm to BayesOpt~\cite{martinez2014}. 
The numerical results are obtained by simulating the OCT transfer function as described in \cite{Verstraete15proc,Verstraete14} and maximizing the OCT signal.
The input dimension for this example is three. Three Zernike aberrations are considered, namely the defocus and two astigmatisms. 
These are generally the largest optical wavefront aberrations in the human eye.
The noise of a real OCT signal is approximated by adding Gaussian white noise with a standard deviation of 0.01.
\begin{figure}[htbp]
\centerline{\includegraphics[width=0.85\columnwidth]{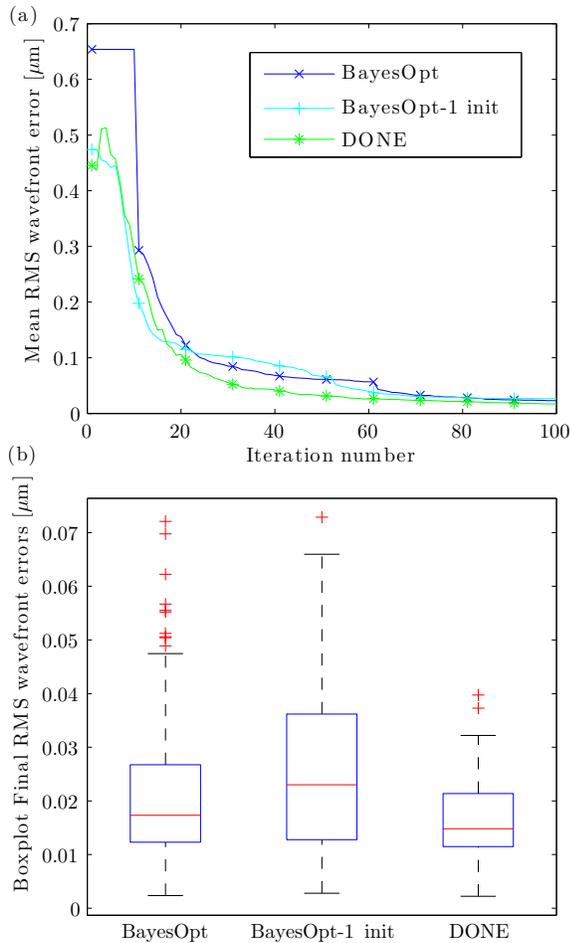}}
\caption{(a) The RMS wavefront error of DONE and BayesOpt averaged over 100 simulations versus the number of iterations. (b) A boxplot of 100 final RMS wavefront errors after 100 iterations for DONE and BayesOpt.On each box, the central line is the median, the edges of the box are the 25th and 75th percentiles, and the whiskers extend to the most extreme data points not considered outliers. Outliers are plotted individually. }
\label{fig:DONEvsBayesOptOCT}
\end{figure}
The results are shown in Fig.~\ref{fig:DONEvsBayesOptOCT}. 
For the DONE algorithm the same parameters are used as described in \cite{Verstraete15}, only $\reg$ is chosen to be equal to $3$.
The number of cosines $D = 1000$ is chosen as large as possible such that the computation time still remains around $1$ ms. This is sufficiently fast to keep up with modern OCT B-scan acquisition and processing rates. 
The DONE algorithm is compared to BayesOpt with the default parameters and to BayesOpt with only one instead of $10$ prior measurements, the latter is referred to as BayesOpt-1 init. 
Other values for the parameters of BayesOpt, obtained with trial and error, did not result in a significant performance increase.
To use the BayesOpt algorithm, the inputs had to be normalized between 0 and 1. 
For each input aberration, the region -0.45 $\mu$m to 0.45 $\mu$m was scaled to the region 0 to 1.
The results for BayesOpt and DONE are very similar.
The mean error of the DONE algorithm is slightly lower than the BayesOpt algorithm.
However, the total average computation time for the DONE algorithm was $93$ ms, while the total average computation time of Bayesopt was $1019$ ms. 

\subsection{Tuning of an Optical Beam-forming Network}

In wireless communication systems, optical beam-forming networks (OBFNs) can be used to steer the reception or transmission angle of a phased array antenna~\cite{hansen2009phased} in the desired direction. 
In the case of reception, the signals that arrive at the different antenna elements of the phased array are combined in such a way that positive interference of the signals occurs only in a specific direction. A device based on optical ring resonators~\cite{roeloffzen2005ring} (ORRs) that can perform this signal processing technique in the optical domain was proposed in~\cite{meijerink2010novel}. This OBFN can provide accurate control of the reception angle in broadband wireless receivers.

To achieve a maximal signal-to-noise ratio (SNR), the actuators in the OBFN need to be adapted according to the desired group delay of each OBFN path, which can be calculated from the desired reception angle. Each ORR is controlled by two heaters that influence its group delay, however the relation between heater voltage and group delay is nonlinear. Even if the desired group delay is available, controlling the OBFN comes down to solving a nonlinear optimization problem. Furthermore, the physical model of the OBFN can become quite complex if many ORRs are used, and the models are prone to model inaccuracies. Therefore, a black-box approach like in the DONE algorithm could help in the tuning of the OBFN. Preliminary results using RFEs in an offline fashion on this application can be found in~\cite{Bliek2015166}. Here, we demonstrate the advantage of online processing in terms of performance by using DONE instead of the offline algorithm in~\cite{Bliek2015166}.

An OBFN simulation based on the same physical models as in~\cite{Bliek2015166} will be used in this section, with the following differences: 1) the implementation is done in C++; 2) ORR properties are equal for each ORR; 3) heater voltages with offset and crosstalk~\cite[Appendix B]{zhuang2010ring} have been implemented; 4) a small region outside the bandwidth of interest has a desired group delay of $0$; 5) an $8\times 1$ OBFN with $12$ ORRs is considered; 6) the standard deviation of the measurement noise was set to $7.5 \cdot 10^{-3}$. The input of the simulation is the normalized heater voltage for each ORR, and the output is the corresponding mean square error of the difference between OBFN path group delays and desired delays. The simulation contains $24$ heaters (two for each ORR, namely one for the phase shift and one for the coupling constant), making the problem $24$-dimensional. Each heater influences the delay properties of the corresponding ORR, and together they influence the OBFN path group delays.

The DONE algorithm was used on this simulation to find the optimal heater voltages. The number of basis functions was $D=6000$, which was the lowest number that gave an adequate performance. The p.d.f. $p_{\freqRV}$ was a normal distribution with variance $0.5$. 
The regularization parameter was $\lambda=0.1$. 
The exploration parameters were $\sigma_{\perturb}=\sigma_{\xi}=0.01$. 
In total, $3000$ measurements were taken.

Just like in the previous application, the DONE algorithm was compared to the Bayesian optimization library BayesOpt~\cite{martinez2014}. The same simulation was used in both algorithms, and BayesOpt also had $3000$ function evaluations available. The other parameters for BayesOpt were set to their default values, except for the noise parameter which was set to $0.1$ after calculating the influence of the measurement noise on the objective function. Also, in-between hyper-parameter optimization was turned off after noticing it did not influence the results while being very time-consuming.

The results for both algorithms are shown in Fig.~\ref{fig:DONEvsBayesOptOBFN}. The found optimum at each iteration is shown for the two algorithms. For DONE, the mean of $10$ runs is shown, 
while for BayesOpt only one run is shown because of the much longer computation time. The dotted line represents an offline approach: it is the average of $10$ runs of a similar procedure as in~\cite{Bliek2015166}, where a RFE with the same hyper-parameters as in DONE was fitted to $3000$ random measurements and then optimized. The figure clearly shows the advantage of the online approach: because measurements are only taken in regions where the objective function is low, the RFE model can become very accurate in this region. The figure also shows that DONE outperforms BayesOpt for this application in terms of accuracy. On top of that, the total computation time shows a big improvement: one run of the DONE algorithm took less than $2$ minutes, while one run of BayesOpt took $5800$ minutes. 

The big difference in computation time for the OBFN application can be explained by looking at the total number of measurements $\N$. 
Even though the input dimension is high compared to the other problems, $\N$ is the main parameter that causes BayesOpt to slow down for a large number of measurements. 
This is because the models used in Bayesian optimization typically depend on the kernel matrix of all samples, which will increase in size each iteration. 
The runtime for one iteration of the DONE algorithm is, in contrast, independent of the number of previous measurements.

\begin{figure}[htbp]
\centerline{\includegraphics[width=0.9\columnwidth]{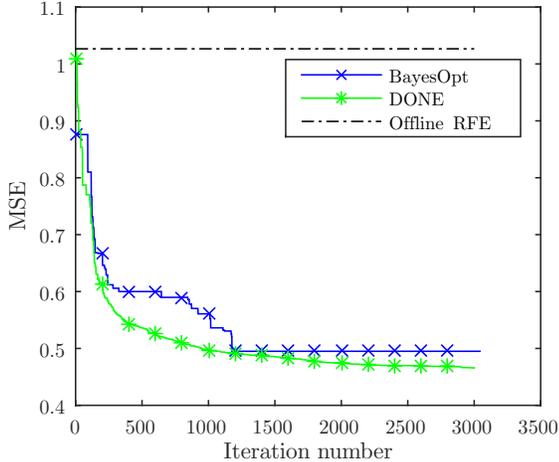}}
\caption{The mean square error of DONE and BayesOpt applied to the OBFN application, plotted versus the number of iterations. For DONE, the values are averaged over $10$ runs. For BayesOpt only 1 run is shown. The dotted line is the result of fitting a RFE using $3000$ random measurements and optimizing that RFE, averaged over $10$ runs.}
\label{fig:DONEvsBayesOptOBFN}
\end{figure}

\subsection{Robot Arm Movement}\label{sec:robotarm}
The previous two examples have illustrated how the DONE algorithm outperforms BayesOpt in terms of speed (both OCT and OFBN) and how its online processing scheme reduces the number of required measurements compared to offline processing (OFBN), respectively. The dimensions in both problems were three and 27, respectively, which is still relatively modest. To illustrate that DONE also works in higher dimensions, we will now consider a toy example from robotics. The following model of a three-link-planar robot, which has been adapted from~\cite{de2009method}, is considered:
\begin{align}
a_i(k) & = u_i(k) + \sin\left(\pi/180\sum_{j=1}^i \alpha_j(k-1)\right)\cdot 9.8\cdot 0.05,\\
v_i(k) & = v_i(k-1) + a_i(k),\\
\alpha_i(k) & = \alpha_i(k-1) + v_i(k),\\
x(k) & = \sum_{j=1}^3 l_j\cos\left(\pi/2 + \pi/180\sum_{j=1}^i \alpha_j(k)\right),\\
y(k) & = \sum_{j=1}^3 l_j\sin\left(\pi/2 + \pi/180\sum_{j=1}^i \alpha_j(k)\right).
\end{align}
Here, $\alpha_i(k)$ represents the angle in degrees of link $i$ at time step $k$, $v_i(k)$ and $a_i(k)$ are the first and second derivative of the angles, $u_i(k)\in [-1, 1]$ is the control input, $x(k)$ and $y(k)$ denote the position of the tip of the arm, and $l_1 = l_2 = 8.625$ and $l_3 = 6.125$ are the lengths of the links. The variables are initialized as $a_i(0) = v_i(0) = \alpha_i(0) = 0$ for $i=1,2,3$.
We use the DONE algorithm to design a sequence of control inputs $u_i(1),\dots,u_i(50)$ such that the distance between the tip of the arm and a fixed target at location $(6.96,12.66)$ at the $50$-th time step is minimized. 
The input for the DONE algorithm is thus a vector containing $u_i(k)$ for $i=1,2,3$ and $k=1,\ldots, 50$.
This makes the problem $150$-dimensional.
The output is the distance between the tip and the target at the $50$-th time step.
The initial guess for the algorithm was set to a random control sequence with a uniform distribution over the set $[-1, 1]$ for each robot arm $i$. 
We would like to stress that this example has been chosen for its high-dimensional input. We do not consider this approach a serious contender for specialized control methods in robotics.

The hyper-parameters for the DONE algorithm were chosen as follows.
The number of basis functions was $D=3000$, which was the lowest number that gave consistent results.
The regularization parameter was $\lambda=10^{-3}$.
The p.d.f. $p_{\freqRV}$ was set to a normal distribution with variance one. 
The exploration parameters were set to $\sigma_{\perturb}=\sigma_{\xi}=5\cdot 10^{-5}$.
The number of measurements $N$ was set to $10 000$.

No comparison with other algorithms has been made for this application. 
The computation time of the Bayesian optimization algorithm scales with the number of measurements and would be too long with $10000$ measurements, as can be seen in Table~\ref{tab:times}.
Algorithms like reinforcement learning use other principles, hence no comparison is given.
Our main purpose with this application is to demonstrate the applicability of the DONE algorithm to high-dimensional problems.
\begin{figure}[htbp]
\centerline{\includegraphics[width=0.9\columnwidth]{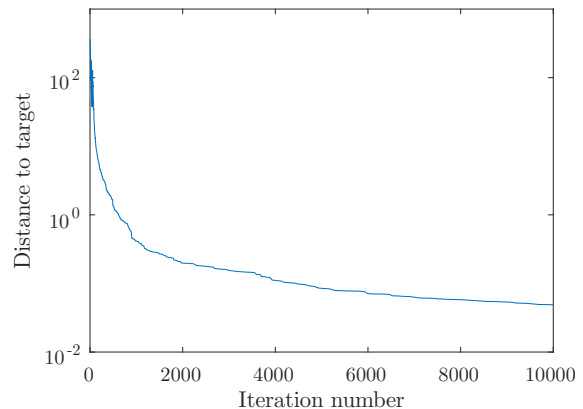}}
\caption{The mean distance to target for the robot arm at time step $50$, after minimizing this distance with DONE, plotted versus the number of iterations, averaged over $10$ runs.}
\label{fig:robotarm}
\end{figure}
Figure~\ref{fig:robotarm} shows the distance to the target at time step $50$ for different iterations of the DONE algorithm, averaged over $10$ runs with different initial guesses.
The control sequences converge to a sequence for which the robot arm goes to the target, i.e., DONE has successfully been applied to a problem with a high input dimension. 
The number of basis functions required did not increase when compared to the other applications in this paper, although more measurements were required.
The computation time for this example and the other examples is shown in Table~\ref{tab:times}.

\begin{table}[htbp]
\renewcommand{\arraystretch}{1.3}
\caption{Computation Time: DONE vs BayesOpt}
\label{tab:times}
\centering
\begin{tabular}{|c|c|c|c|c|c|}
\hline
Problem & Method & Input dim. & $\N$ & $\D$ & Time (s) \\
\hline
\multirow{2}{*}{Camelback} & DONE & $2$ & $100$ & $50$ & 
$0.0683$ \\
& BayesOpt & $2$ & $100$ & - & $0.3049$ \\
\hline
\multirow{2}{*}{OCT} & DONE & $3$ & $100$ & $1000$ & $0.093$ \\
& BayesOpt & $3$ & $100$ & - & $1.019$ \\
\hline
\multirow{2}{*}{OBFN} & DONE & $24$ & $3000$ & $6000$ & 
$99.7$ \\
& BayesOpt & $24$ & $3000$ & - & $3.48\cdot 10^5$ \\
\hline
\multirow{1}{*}{Robot arm} & DONE & $150$ & $10000$ & $3000$ & $99.1$ \\
\hline
\end{tabular}
\end{table}

\section{Conclusions}\label{sec:conclusion}

We have analyzed an online optimization algorithm called DONE that is used to find the minimum of a function using measurements that are costly and corrupted by noise.
DONE maintains a surrogate model in the form of a random Fourier expansion (RFE), which is updated whenever a new measurement is available, and minimizes this surrogate with standard derivative-based methods.
This allows to measure only in regions of interest, reducing the overall number of measurements required.
The DONE algorithm is comparable to Bayesian optimization algorithms, but it has the distinctive advantage that the computational complexity of one iteration does not grow with the number of measurements that have already been taken.

As a theoretical result, we have shown that a RFE that is trained with linear least squares can approximate square integrable functions arbitrarily well, with high probability. An upper bound on the regularization parameter used in this training procedure was given, as well as an optimal and a more practical probability distribution for the parameters that are chosen randomly. 
We applied the DONE algorithm to an analytic benchmark problem and to three applications: optical coherence tomography, optical beam-forming network tuning, and a robot arm.
We compared the algorithm to BayesOpt, a Bayesian optimization library.
The DONE algorithm gave accurate results on these applications while being faster than the Bayesian optimization algorithm, due to the fixed computational complexity per iteration.

\appendices

\section{Proof of convergence of the least squares solution}\label{ap:proofLStheorem}
In this section, we show that using the least squares solution in the RFE gives a function that approximates the true unknown function $\func$.
To prove this, we make use of the results in~\cite{rahimi2008uniform} and of \cite[Thm. 2]{jennrich1969asymptotic} and \cite[Key Thm.]{vapnik1999overview}.
\begin{proof}[Proof of Theorem~\ref{thm:LSconv}]
Let the constant $m>0$ be given by
\begin{align}
& m = \left|\left|	\left(\frac{1}{N} \A_N^T \A_N + \frac{\reg}{N}  \mathbf{I}_{\D\times\D}\right)^{-1}\frac{1}{N} \A_N^T \mathbf y_N\right|\right|_2,
\end{align}
and define the set
 $   	C_m = \{\mathbf{\ampl} \in \mathbb R^{\D}: \ ||\mathbf{\ampl}||_2 \leq m\}$.
	Note that $C_m$ is a compact set.
    The least squares weight vector
\begin{align}
    \mathbf \amplLS
    & = \left(\A_N^T \A_N + \reg \mathbf{I}_{\D\times\D}\right)^{-1}\A_N^T \mathbf y_N\nonumber\\
    & = \left(\frac{1}{N}\A_N^T \A_N + \frac{\reg}{N}\mathbf{I}_{\D\times\D}\right)^{-1}\frac{1}{N}\A_N^T \mathbf y_N,
\end{align}
is also the solution to the constrained, but unregularized least squares problem (see \cite[Sec. 12.1.3]{golub2012matrix})
\begin{align}
    \mathbf \amplLS = \argmin_{\mathbf{\ampl} \in C_m}
 &  \frac{1}{N}||\mathbf y_N - \A_N \mathbf \ampl||_2^2.
\end{align}
Now, note that a decrease in $\reg$ leads to an increase in $m$. Since  $\reg/N \leq \Lambda$ by assumption and the upper bound $\Lambda$ in Theorem~\ref{thm:LSconv} satisfies 
\begin{align}
& \left|\left|	\left(\frac{1}{N} \A_N^T \A_N +  \Lambda \  \mathbf{I}_{\D\times\D}\right)^{-1}\frac{1}{N} \A_N^T \mathbf y_N\right|\right|_2
= M,\\
&M = \sqrt{\sum_{\k=1}^{\D}\left(\frac{\amplTH(\freq_{\k},\b_{\k})}{(2\pi)^{\d}\D p_{\freqRV}(\freq_{\k})p_{\B}(\b_{\k})}\right)^2},
\end{align}
we have that $m\geq M$.
We will need this lower bound on $m$ to make use of the results in~\cite{rahimi2008uniform} later on in this proof.

Recall from Section~\ref{sec:LSconv} that the vector $\mathbf{y}_N$ depends on the function evaluations and on measurement noise $\eta$ that is assumed to be zero-mean and of finite variance $\sigma_{\etaRV}^2$.
    We first consider the noiseless case, i.e. $y_n = \func(\input_n)$.
    For $\input \in \inputspace$, $\mathbf{\ampl}\in \mathbb R^{\D},$ let 
    \begin{align}
    E(\input, \mathbf{\ampl}) & =  \func(\input) -  \sum_{\k=1}^{\D} \ampl_{\k} \cos(\freq_{\k}^T \input + \b_{\k}).
    \end{align}
    Using the Cauchy-Schwarz inequality, we have the following bound for all $\input \in \inputspace$, $\mathbf {\ampl} \in C_m$:
    \begin{align}
    	E(\input, \mathbf{\ampl})^2 
        & = \func(\input)^2 +  \left( \sum_{\k=1}^{\D} \ampl_{\k} \cos(\freq_{\k}^T \input + \b_{\k})\right)^2 
        \nonumber\\ & \quad
        - 2\func(\input) \sum_{\k=1}^{\D} \ampl_{\k} \cos(\freq_{\k}^T \input + \b_{\k})  \nonumber\\
        & \leq \func(\input)^2 +  \left( \sum_{\k=1}^{\D} \ampl_{\k} \cos(\freq_{\k}^T \input + \b_{\k})\right)^2 
        \nonumber\\ & \quad
        + 2\left|\func(\input)\right| \left|  \sum_{\k=1}^{\D} \ampl_{\k} \cos(\freq_{\k}^T \input + \b_{\k}) \right| \nonumber\\
        & \leq \func(\input)^2 +  \sum_{\k=1}^{\D} \left|\ampl_{\k}\right|^2 + 2\left|\func(\input)\right|  \sqrt{\sum_{\k=1}^{\D}\left| \ampl_{\k}  \right|^2} \nonumber\\
        & \leq \func(\input)^2 + m^2 + 2\func(\input) m\nonumber\\
        & \leq \left(||\func||_{\infty} + m\right)^2
        .\label{eq:boundforE}
	\end{align}
    Note that $E(\input, \mathbf{\ampl})$ is continuous in $\mathbf{\ampl}$ and measurable in $\input$. Let now $\inputRV_{n}$ denote i.i.d. random vectors with distribution $p_{\inputRV}$. Using Theorem~\cite[Thm. 2]{jennrich1969asymptotic} we get, with probability one,
    \begin{align}
    	\lim_{N\rightarrow \infty} \sup_{\mathbf \ampl \in C_m} \left|\frac{1}{N} \sum_{n=1}^N E(\inputRV_n, \mathbf{\ampl})^2 - \int_{\inputspace} E(\input, \mathbf{\ampl})^2 p_{\inputRV}(\input) d\input \right| = 0.
    \end{align}
    Since almost sure convergence implies convergence in probability~\cite[Ch. 2]{van2000asymptotic}, we also have:
    \begin{align}
    	& \lim_{N\rightarrow \infty} P\left( \sup_{\mathbf \ampl \in C_m} \left|\frac{1}{N} \sum_{n=1}^N E(\inputRV_n, \mathbf{\ampl})^2
        \right.\right. \nonumber\\ & \qquad \left.\left. 
        - \int_{\inputspace} E(\input, \mathbf{\ampl})^2 p_{\inputRV}(\input) d\input \right| > \epsilon \right) = 0 \ \forall \epsilon > 0. \label{eq:convprob}
    \end{align}
We will need this result when considering the case with noise.  
For the case with noise, i.e.  $y_n = \func(\input_n)+\eta_n$, let 
    \begin{align}\label{eq:defEtilde}
    \tilde E(\input, \eta, \mathbf{\ampl})^2 & =  \left(\func(\input) + \eta - \sum_{\k=1}^{\D} \ampl_{\k} \cos(\freq_{\k}^T \input + \b_{\k}) \right)^2\nonumber\\
    & = E(\input,\mathbf{\ampl})^2  + 2\eta E(\input,\mathbf{\ampl}) + \eta^2
    .
    \end{align}
    Using the properties of the noise $\eta$ with p.d.f. $p_{\etaRV}$, this gives the following mean square error:
    \begin{align}
    & \int_{\mathbb R} \int_{\inputspace} \tilde E(\input,\eta,\mathbf{\ampl})^2 p_{\inputRV}(\input) p_{\etaRV}(\eta)d\input d\eta \nonumber\\
    & = \int_{\inputspace}E(\input,\mathbf{\ampl})^2 p_{\inputRV}(\input) \left(\int_{\mathbb R} p_{\etaRV}(\eta) d\eta\right) d\input 
    \nonumber\\ & \quad
    +  2\int_{\inputspace} E(\input,\mathbf{\ampl}) \left(\int_{\mathbb R} \eta p_{\etaRV}(\eta) d\eta \right) p_{\inputRV}(\input) d\input 
    \nonumber\\ & \quad
    + \int_{\inputspace} p_{\inputRV}(\input) \left( \int_{\mathbb R}\eta^2 p_{\etaRV}(\eta) d\eta \right) d\input\nonumber\\
    & =     \int_{\inputspace}E(\input,\mathbf{\ampl})^2 p_{\inputRV}(\input) d\input + \int_{\inputspace} E(\input,\mathbf{\ampl}) \underbrace{\mathbb E[\etaRV_n]}_{=0}p_{\inputRV}(\input)d\input \nonumber\\
    & \quad  + \mathbb E[\etaRV_n^2]\nonumber\\
    & = \int_{\inputspace}E(\input,\mathbf{\ampl})^2 p_{\inputRV}(\input) d\input + \sigma_{\etaRV}^2.
    \label{eq:Etildeint}
    \end{align}
    Here, $\etaRV_n$ is a random variable with distribution $p_{\etaRV}$.
	For any choice of $\epsilon_0, \epsilon_1, \epsilon_2, \epsilon_3 >0$ such that $\epsilon_1 + \epsilon_2 + \epsilon_3 = \epsilon_0$, we have, following a similar proof as in~\cite[Thm. 3.3(a)]{beitollahi2012convergence}:
    \begin{align}
    & \ P\left( \sup_{\mathbf \ampl \in C_m} \left|\frac{1}{N} \sum_{n=1}^N \tilde E(\inputRV_n, \etaRV_n, \mathbf{\ampl})^2 -
      \right.\right. \nonumber\\ & \qquad \left.\left.
    \int_{\inputspace} \int_{\mathbb R} \tilde E(\input, \eta, \mathbf{\ampl})^2 p_{\inputRV}(\input) p_{\etaRV}(\eta) d\input d\eta \right| > \epsilon_0 \right) \nonumber\\
    = &  \ P\left( \sup_{\mathbf \ampl \in C_m} \left|\frac{1}{N} \sum_{n=1}^N E(\inputRV_n, \mathbf{\ampl})^2
    + \frac{2}{N} \sum_{n=1}^N \etaRV_n E(\inputRV_n, \mathbf{\ampl})
    \right.\right. \nonumber\\ & \qquad \left.\left.
    + \frac{1}{N} \sum_{n=1}^N \etaRV_n^2
- \int_{\inputspace}  E(\input, \mathbf{\ampl})^2 p_{\inputRV}(\input) d\input  - \sigma_{\etaRV}^2 \right| > \epsilon_0 \right) \nonumber\\  
    \leq & \ P\left( \sup_{\mathbf \ampl \in C_m}\left\{ \left|\frac{1}{N} \sum_{n=1}^N E(\inputRV_n, \mathbf{\ampl})^2 - \int_{\inputspace}  E(\input, \mathbf{\ampl})^2 p_{\inputRV}(\input) d\input \right|
    \right.\right. \nonumber\\ & \  \left.\left.
    +\left|\frac{2}{N} \sum_{n=1}^N \etaRV_n E(\inputRV_n, \mathbf{\ampl})\right|
    +
    \left|\frac{1}{N} \sum_{n=1}^N \etaRV_n^2 - \sigma_{\etaRV}^2\right| \right\}   > \epsilon_0 \right) \nonumber\\
    \leq & \ P\left( \sup_{\mathbf \ampl \in C_m} \left|\frac{1}{N} \sum_{n=1}^N E(\inputRV_n, \mathbf{\ampl})^2 - \int_{\inputspace}  E(\input, \mathbf{\ampl})^2 p_{\inputRV}(\input) d\input \right|
    \right. \nonumber\\ & \  \left.
    +\sup_{\mathbf \ampl \in C_m}\left|\frac{2}{N} \sum_{n=1}^N \etaRV_n E(\inputRV_n, \mathbf{\ampl})\right|
    +
    \left|\frac{1}{N} \sum_{n=1}^N \etaRV_n^2 - \sigma_{\etaRV}^2\right|    > \epsilon_0 \right) \nonumber\\
    \leq & \  P\left( \sup_{\mathbf \ampl \in C_m} \left|\frac{1}{N} \sum_{n=1}^N E(\inputRV_n, \mathbf{\ampl})^2 - \int_{\inputspace}  E(\input, \mathbf{\ampl})^2 p_{\inputRV}(\input) d\input \right| > \epsilon_1
    \right. \nonumber\\ & \qquad \left.
    \mathrm{\ or\ } \sup_{\mathbf \ampl \in C_m}\left|\frac{2}{N} \sum_{n=1}^N \etaRV_n  E(\inputRV_n, \mathbf{\ampl})\right| > \epsilon_2
\right. \nonumber\\ & \qquad \left.
        \mathrm{\ or\ } 
    \left|\frac{1}{N} \sum_{n=1}^N \etaRV_n^2 - \sigma_{\etaRV}^2\right|    > \epsilon_3 \right) \nonumber\\
    \leq & \  P\left( \sup_{\mathbf \ampl \in C_m} \left|\frac{1}{N} \sum_{n=1}^N E(\inputRV_n, \mathbf{\ampl})^2 - \int_{\inputspace}  E(\input, \mathbf{\ampl})^2 p_{\inputRV}(\input) d\input \right| > \epsilon_1 \right)
     \nonumber\\ & \  
    + P\left(\sup_{\mathbf \ampl \in C_m}\left|\frac{2}{N} \sum_{n=1}^N \etaRV_n   E(\inputRV_n, \mathbf{\ampl})\right| > \epsilon_2
    \right)
    \nonumber\\
    & \ + P\left(
    \left|\frac{1}{N} \sum_{n=1}^N \etaRV_n^2 - \sigma_{\etaRV}^2\right|    > \epsilon_3 \right). \nonumber  
\label{eq:3convs}
    \end{align}
    Of these last three probabilities, the first one is proven to converge to zero in~\eqref{eq:convprob}, while the last one converges to zero by the weak law of large numbers. For the second probability, we can make use of Theorem~\cite[Thm. 2]{jennrich1969asymptotic} again, noting that $\eta_n   E(\input_n, \mathbf{\ampl})$ is continuous in $\mathbf{\ampl}$. We use~\eqref{eq:boundforE} to get
     \begin{align}
     	\left|\eta   E(\input, \mathbf{\ampl})\right| \leq |\eta| \left(||\func||_{\infty} + m\right) \ \forall \input, \eta, \mathbf{\ampl}.
 	\end{align}
 Again, since uniform convergence implies convergence in probability, and since $\mathbb E[\etaRV_n E(\inputRV_n, \mathbf{\ampl})] = \mathbb E[\etaRV_n] \mathbb E[E(\inputRV_n, \mathbf{\ampl})] = 0$ for all $n$, using Theorem~\cite[Thm. 2]{jennrich1969asymptotic} gives the desired convergence in probability
\begin{align}
 \lim_{N\rightarrow \infty} P\left(\sup_{\mathbf{\ampl} \in C_m } \left|\frac{1}{N} \sum_{n=1}^N \etaRV_n   E(\inputRV_n, \mathbf{\ampl})\right| > \epsilon_2
     \right) = 0 \ \forall \epsilon_2 .
\end{align}
Together with the other two convergences and~\eqref{eq:3convs} we get:
\begin{align}
& \lim_{N\rightarrow \infty} P\left( \sup_{\mathbf \ampl \in C_m} \left|\frac{1}{N} \sum_{n=1}^N \tilde E(\inputRV_n, \etaRV_n, \mathbf{\ampl})^2
\right.\right.\nonumber\\ & \qquad  \left.\left. 
- \int_{\mathbb R} \int_{\inputspace} \tilde E(\input, \eta, \mathbf{\ampl})^2 p_{\inputRV}(\input) p_{\etaRV}(\eta) d\input d\eta \right| > \epsilon \right) = 0.
\end{align}

The following bound follows from~\eqref{eq:boundforE} and~\eqref{eq:Etildeint}:
    \begin{align}
    	0 & \leq \int_{\mathbb R} \int_{\inputspace} \tilde E(\input,\eta,\mathbf{\ampl})^2 p_{\inputRV}(\input) p_{\etaRV}(\eta)d\input d\eta  \nonumber\\
        & \leq \left(||\func||_{\infty} + m\right)^2 + \sigma_{\etaRV}^2.
	\end{align}
    In light of this bound, \cite[Key Thm.]{vapnik1999overview} now implies that the mean square error between the output of the RFE with least squares weight vector and the noisy meansurements is approaching its ideal value as the number of samples increases. More precisely, for any choice of $\epsilon_4 > 0$ and $\delta_1 > 0$, there exists an $N_0$ such that, for all $N > N_0$,
\begin{align}
&	\left|   \int_{\mathbb R} \int_{\inputspace}  \tilde E(\input, \eta, \mathbf{\amplLSRV})^2 p_{\inputRV}(\input) p_{\etaRV}(\eta)d\input d\eta
\right.    \nonumber\\ & \qquad \left.
    - \int_{\mathbb R} \int_{\inputspace} \tilde E(\input, \eta, \mathbf{\Ampl^0})^2 p_{\inputRV}(\input) p_{\etaRV}(\eta) d\input d\eta \right| < \epsilon_4 \label{eq:tri1}
\end{align}
with probability at least $1-\delta_1$.
Here, $\mathbf{\amplLSRV}$ denotes the vector $\mathbf{\amplLS}$ as a random variable as it depends on the input and noise samples
and on the samples $\freq_1,\ldots, \freq_{\D}, \b_1, \ldots, \b_{\D}$,
and $\mathbf{\Ampl^0}\in C_m$ minimizes $\int_{\mathbb R} \int_{\inputspace} \tilde E(\input, \eta, \mathbf \ampl) p_{\inputRV}(\input) p_{\etaRV}(\eta)d\input d\eta$.
Next, it is shown that the same holds for the mean square error between the least-squares RFE outputs and the unknown, noise-free function values.

According to~\cite[Thm 3.2]{rahimi2008uniform}, for any $\delta_2>0$, with probability at least $1-\delta_2$ w.r.t. $\freqRV_1,\dots,\freqRV_{\D}$ and $\B_1,\dots,\B_{\D}$, there exists a $\mathbf{\ampl} \in C_m$ with the following bound%
\footnote{The weights found in the proof of the cited theorem satisfy $\mathbf{\ampl} \in C_m$ if $m\geq M$, which was shown in the beginning of this appendix. Here we also made use of the result from Theorem~\ref{thm:zeroerror} of this paper to get what is denoted with $\alpha$ in~\cite{rahimi2008uniform}. We have also used, with the notation of~\cite{rahimi2008uniform}, that $||f-\hat f||_{\mu}\leq ||f-\hat f||_{\infty}$.}%
:
	\begin{align} \label{eq:gammadef}
&\int_{\inputspace}\left(\func(\input) - \sum_{\k=1}^{\D} \ampl_{\k} \cos(\freqRV_{\k}^T \input + \B_{\k})\right)^2 p_{\inputRV}(\input)d\input   < \frac{\gamma(\delta_2)^2}{\D},\nonumber\\
&\gamma(\delta_2) = \sup_{\freq,\b} \left|\frac{1}{(2\pi)^{\d}}\frac{\amplTH(\freq,\b)}{p_{\freqRV}(\freq) p_{\B}(\b)}\right| \left(\sqrt{\log \frac{1}{\delta_2}} + 4r\right),\nonumber\\
&r = \sup_{\input \in \inputspace} ||\input||_2 \sqrt{\sigma^2 \d + \pi^2/3},
	\end{align}
    with $\sigma^2$ denoting the variance of $p_{\freqRV}$.
For this particular $\mathbf{\ampl}$, \eqref{eq:defEtilde},~\eqref{eq:Etildeint} and \eqref{eq:gammadef} imply that
\begin{align}
\int_{\mathbb R} \int_{\inputspace} \tilde E(\input,\eta,\mathbf{\ampl})^2 p_{\inputRV}(\input) p_{\etaRV}(\eta) d\input d\eta   < \frac{\gamma(\delta_2)^2}{\D} + \sigma_{\etaRV}^2.
\end{align}

Since $\mathbf{\Ampl^0} \in C_m$ minimizes the left-hand in the equation above by definition,
we also have that
	\begin{align}
	\int_{\mathbb R} \int_{\inputspace}  \tilde E(\input, \eta, \mathbf{\Ampl^0})^2 p_{\inputRV}(\input) p_{\etaRV}(\eta) d\input d\eta  & 
    < \frac{\gamma(\delta_2)^2}{\D} + \sigma_{\etaRV}^2\label{eq:tri2}
    \end{align}
    with probability at least $1-\delta_2$.    
     Since the event in \eqref{eq:tri2} only depends on $\freqRV_1,\dots,\freqRV_{\D}$ and $\B_1,\dots,\B_{\D}$, while the event in \eqref{eq:tri1} only depends on the input and noise samples, 
     we can combine these two equations as follows.
    For any choice of $\epsilon_4 > 0$, $\delta_1 > 0$ and $\delta_2 > 0$, there exists an $N_0$ such that, for all $N > N_0$,
    \begin{align}
    \int_{\mathbb R} \int_{\inputspace}  \tilde E(\input, \eta, \mathbf{\amplLSRV})^2 p_{\inputRV}(\input) p_{\etaRV}(\eta)d\input d\eta < \epsilon_4 + \frac{\gamma(\delta_2)^2}{\D} + \sigma_{\etaRV}^2
    \end{align}
    with probability at least $(1-\delta_1)(1-\delta_2)$.   
    Using~\eqref{eq:Etildeint} now gives the following result.
    For any choice of $\epsilon_4 > 0$, $\delta_1 > 0$ and $\delta_2 > 0$, there exists an $N_0$ such that, for all $N > N_0$, we have
    \begin{align}
    \int_{\inputspace} E(\input, \mathbf{\amplLSRV})^2 p_{\inputRV}(\input)d\input < \epsilon_4 + \frac{\gamma(\delta_2)^2}{\D}
    \end{align}
    with probability at least $(1-\delta_1)(1-\delta_2)$.
    
    Choosing $\D_0, \epsilon_4, \delta_1$ and $\delta_2$ such that $\D_0 > \gamma(\delta_2)^2/(\epsilon-\epsilon_4)$ and $(1-\delta_1)(1-\delta_2) = \delta$ concludes the proof.
    
\end{proof}

\section{Minimum-variance properties}
\label{app:pw}

The following theorem presents the probability density function for $\freqRV_{\k}$ that minimizes the variance of a RFE at a fixed measurement location $\input$.

\begin{theorem}\label{thm:minvar}

Given $\input$, the p.d.f. $p_{\freqRV}^*$ that minimizes the variance of the unbiased estimator $\RFERV(\input) = \sum_{\k=1}^\D C_{\k} \cos(\freqRV_{\k}^T \input + \B_{\k})$ as defined in Theorem~\ref{def:RFE}, with 
$\Ampl_{\k}$ as defined in Theorem~\ref{thm:unbiased},
is equal to 
\begin{align}\label{eq:optpdf}
p_{\freqRV}^*(\freq) & = \frac{|\hat \func(\freq)|\sqrt{\cos(2\angle \hat\func(\freq)+2\freq^T\input) + 2}}{\int_{\mathbb R^{\d}} |\hat \func(\tilde\freq)| \sqrt{\cos(2\angle \hat\func(\tilde\freq)+2\tilde\freq^T\input) + 2} d\tilde\freq}.
\end{align}    
For this choice of $p_{\freqRV}$, the variance is equal to
\begin{align}
& \frac{1}{2 \D(2\pi)^{2\d}}  \left(\int_{\mathbb R^{\d}}|\hat \func(\freq)| \sqrt{\cos(2\angle \hat\func(\freq)+2\freq^T\input) + 2} d\freq\right)^2\nonumber\\
& - \func(\input)^2.
\end{align}
\end{theorem}

\begin{proof}
The proof
is similar to the proof of~\cite[Thm. 4.3.1]{rubinstein2011simulation}. 
Let $q_{\freqRV}$ be any p.d.f. of $\freqRV_{\k}$ that satisfies $q_{\freqRV}(\freq)>0$ if $|\hat \func(\freq)|>0$. 
Let $\mathrm{Var}_{q_{\freqRV},p_{\B}}$ be the variance of $\RFERV(\input)$ under the assumption that $p_{\freqRV} = q_{\freqRV}$, $p_{\B} = \mathrm{Uniform}(0,2\pi)$, and $C_{\k}
= \frac{2}{\D(2\pi)^{\d}} \frac{|\hat \func(\freqRV_{\k})|}{q_{\freqRV}(\freqRV_{\k})}\cos(\angle \hat \func(\freqRV_{\k}) - \B_{\k})$. 
According to Theorem~\ref{thm:unbiased}, this choice for $\Ampl_{\k}$ makes sure that $\RFERV(\input)$ is an unbiased estimator, i.e., $f(\input)=\mathbb{E}[\RFERV(\input)]$.
The variance of $\RFERV(\input)$ can be computed as:
\begin{align}
& 	\mathrm{Var}_{q_{\freqRV},p_{\B}}[\RFERV(\input)]\nonumber\\
& =
    \mathrm{Var}_{q_{\freqRV},p_{\B}}\left[\sum_{\k=1}^\D \Ampl_{\k} \cos(\freqRV_{\k}^T \input + B_{\k})\right] \nonumber\\
    & = \D \ \mathrm{Var}_{q_{\freqRV},p_{\B}}\left[\Ampl_{1} \cos(\freqRV_1^T \input + B_1)\right]\nonumber\\
    & = 
    \frac{\D}{2\pi} \int_{\mathbb R^{\d}} \int_0^{2\pi}\left(\frac{2}{\D (2\pi)^{\d}} \frac{|\hat \func(\freq)|}{q_{\freqRV}(\freq)}
    \cos(\angle \hat \func(\freq)-\b)\right)^2
    \nonumber\\ & \qquad
    \cos(\freq^T \input + \b)^2 q_{\freqRV}(\freq)d\b d\freq - \func(\input)^2. \label{eq:VarRFE}
\end{align}

For the stated choice of $p_{\freqRV}^*$, 
using 
\begin{align}\label{eq:cosby}
 & \int_0^{2\pi}\cos(\angle \hat \func(\freq) - \b)^2 \cos(\freq^T \input + \b)^2 d\b\nonumber\\
  = & \int_0^{2\pi} \frac{1}{4}(1+\cos(2\angle \hat \func(\freq) - 2\b))(1+\cos(2\freq^T \input + 2\b)) db\nonumber\\
  = & \int_0^{2\pi} \frac{1}{4}d\b + \frac{1}{4} \int_0^{2\pi}\cos(2\angle \hat \func(\freq) - 2\b)d\b 
      \nonumber\\ & 
  +  \frac{1}{4} \int_0^{2\pi}\cos(2\freq^T \input + 2\b)d\b
    \nonumber\\ & 
  + \frac{1}{4} \int_0^{2\pi} \cos(2\angle \hat \func(\freq) - 2\b)\cos(2\freq^T \input + 2\b)d\b \nonumber\\
  = &  \frac{2\pi}{4} + \frac{1}{8}\int_0^{2\pi}\cos(2\angle \hat \func(\freq) + 2\freq^T \input) 
      \nonumber\\ & 
      + \cos(2\angle \hat \func(\freq) -2\freq^T \input- 4\b)d\b\nonumber\\
  = & \frac{2\pi}{4} + \frac{2\pi}{8} \cos(2\angle \hat \func(\freq) +2\freq^T\input)\nonumber\\
 = & \frac{\pi}{4}(\cos(2\angle \hat\func(\freq)+2\freq^T\input) + 2)
\end{align}
we get:
\begin{align}
	& \mathrm{Var}_{p^*_{\freqRV},p_{\B}}[\RFERV(\input)] + \func(\input)^2  = \mathbb{E}_{p^*_{\freqRV},p_{\B}}[\RFERV(\input)^2]\nonumber\\
    & = \frac{\D}{2\pi} \int_{\mathbb R^{\d}} \int_0^{2\pi} \left(\frac{2}{\D(2\pi)^{\d}} \frac{|\hat \func(\freq)|}{p^*_{\freqRV}(\freq)}\cos(\angle \hat \func(\freq) - \b) \right)^2
    \nonumber\\ & \qquad
    \cos(\freq^T \input + \b)^2 p^*_{\freqRV}(\freq)d\b d\freq\nonumber\\
    & = \frac{\D}{2\pi} \int_{\mathbb R^{\d}}  \frac{1}{p^*_{\freqRV}(\freq)}\left(\frac{2}{\D(2\pi)^{\d}}\right)^2 |\hat \func(\freq)|^2 
        \nonumber \\*   & \qquad
    \int_0^{2\pi}\cos(\angle \hat \func(\freq) - \b)^2 \cos(\freq^T \input + \b)^2 d\b d\freq\nonumber\\
    & = \frac{\D}{2\pi} \int_{\mathbb R^{\d}}  \frac{1}{p^*_{\freqRV}(\freq)}\left(\frac{2}{\D(2\pi)^{\d}}\right)^2 |\hat \func(\freq)|^2
    \nonumber\\ & \qquad
    \frac{\pi}{4}(\cos(2\angle \hat\func(\freq)+2\freq^T\input) + 2) d\freq\label{eq:varRFE2}\\
    & \stackrel{\eqref{eq:optpdf}}{=} \frac{\D}{2\pi}   \left(\frac{2}{\D(2\pi)^{\d}}\right)^2
    \nonumber\\ & \qquad
    \left(\int_{\mathbb R^{\d}}|\hat \func(\freq)| \sqrt{\frac{\pi}{4}(\cos(2\angle \hat\func(\freq)+2\freq^T\input) + 2)} d\freq\right)^2\nonumber\\
    & = \frac{1}{2D(2\pi)^{2\d}}   \left(\int_{\mathbb R^{\d}}|\hat \func(\freq)| \sqrt{(\cos(2\angle \hat\func(\freq)+2\freq^T\input) + 2)} d\freq\right)^2
    \end{align}
This gives the value of the optimal variance. To show that the variance is indeed optimal, compare it with any arbitrary p.d.f. $q_{\freqRV}$ using Jensen's
    inequality:
    \begin{align}
    & \mathrm{Var}_{p^*_{\freqRV},p_{\B}}[\RFERV(\input)] + \func(\input)^2\nonumber\\
    & = \frac{\D}{2\pi}   \left(\frac{2}{\D(2\pi)^{\d}}\right)^2 
    \nonumber\\ & \quad
    \left(\int_{\mathbb R^{\d}}\frac{|\hat \func(\freq)|}{q_{\freqRV}(\freq)} \sqrt{\frac{\pi}{4}(\cos(2\angle \hat\func(\freq)+2\freq^T\input) + 2)} q_{\freqRV}(\freq)d\freq\right)^2
    \nonumber\\
    & \stackrel{\mathrm{Jensen}}{\leq}\frac{\D}{2\pi}\left(\frac{2}{\D(2\pi)^{\d}}\right)^2 
        \nonumber\\ & \quad
    \int_{\mathbb R^{\d}} \frac{|\hat \func(\freq)|^2}{q_{\freqRV}(\freq)^2} \frac{\pi}{4}(\cos(2\angle \hat\func(\freq)+2\freq^T\input) + 2) q_{\freqRV}(\freq)d\freq \nonumber\\
& \stackrel{\eqref{eq:cosby}}{=}\frac{\D}{2\pi} \int_{\mathbb R^{\d}} \int_0^{2\pi}\left(\frac{2}{\D (2\pi)^{\d}} \frac{|\hat \func(\freq)|}{q_{\freqRV}(\freq)} \cos(\angle \hat \func(\freq)-\b)\right)^2 \nonumber\\
& \qquad \cos(\freq^T \input + \b)^2 q_{\freqRV}(\freq)d\b d\freq \nonumber\\
& \stackrel{\eqref{eq:VarRFE}}{=} \mathrm{Var}_{q_{\freqRV},p_{\B}}[\RFERV(\input)] + \func(\input)^2.
\end{align}

This shows that the chosen p.d.f. $p^*_{\freqRV}$ gives the minimum variance.

\end{proof}

The following theorem compares the second moments in real and complex RFEs for different probability distributions.

\begin{theorem}\label{thm:Vardistances}
Let $\tilde p_{\freqRV}$, $p^*_{\freqRV}$, $\tilde \RFERV$ and $\RFERV$ be as in Theorems~\ref{thm:minvarcomplex} and \ref{thm:minvar}. Then
\begin{align}
	\frac{1}{\sqrt 3} \mathbb{E}_{p^*_{\freqRV},p_{\B}}[\RFERV(\input)^2] & \leq \mathbb{E}_{\tilde p_{\freqRV},p_{\B}}[\RFERV(\input)^2] \leq \sqrt 3 \ \mathbb{E}_{p^*_{\freqRV},p_{\B}}[\RFERV(\input)^2],\label{eq:minvarcomplexA}\\
     \frac{1}{2}\mathbb{E}_{\tilde p_{\freqRV},p_{\B}}[\tilde \RFERV(\input)^2] & \leq \mathbb{E}_{\tilde p_{\freqRV},p_{\B}}[\RFERV(\input)^2] \leq \frac{3}{2} \ \mathbb{E}_{\tilde p_{\freqRV},p_{\B}}[\tilde \RFERV(\input)^2]. \label{eq:minvarcomplexB}
\end{align}
\end{theorem}

\begin{proof}
From
\begin{align}
	1 & \leq \sqrt{(\cos(2\angle \hat\func(\freq)+2\freq^T\input) + 2)} \leq \sqrt{3},\label{eq:cosineq}
\end{align}
and from~\eqref{eq:optpdf} and~\eqref{eq:optpdfwComplex} it follows that 
\begin{align}
	\frac{1}{\sqrt{3}} p^*_{\freqRV}(\freq) & \leq \tilde p_{\freqRV}(\freq)\leq \sqrt{3} p^*_{\freqRV}(\freq),\nonumber\\
	\frac{1}{\sqrt{3}}\frac{1}{p^*_{\freqRV}(\freq)} & \leq \frac{1}{\tilde p_{\freqRV}(\freq)} \leq \sqrt{3} \frac{1}{p^*_{\freqRV}(\freq)}. \label{eq:probineq}
\end{align}
Combining the above with \eqref{eq:varRFE2} yields:
\begin{align}
	& \frac{1}{\sqrt 3} \mathbb{E}_{p^*_{\freqRV},p_{\B}}[\RFERV(\input)^2] \nonumber\\
    & = \frac{1}{\sqrt 3} 
    \frac{1}{2\D (2\pi)^{2\d}} \nonumber\\
    & \quad
    \int_{\mathbb R^{\d}}  \frac{1}{ p^*_{\freqRV}(\freq)} |\hat \func(\freq)|^2 (\cos(2\angle \hat\func(\freq)+2\freq^T\input) + 2) d\freq \nonumber\\
    & \leq \frac{1}{2\D (2\pi)^{2\d}} \nonumber\\
    & \quad \int_{\mathbb R^{\d}}  \frac{1}{ \tilde p_{\freqRV}(\freq)} |\hat \func(\freq)|^2 (\cos(2\angle \hat\func(\freq)+2\freq^T\input) + 2) d\freq \nonumber\\
    & = \mathbb{E}_{\tilde p_{\freqRV},p_{\B}}[\RFERV(\input)^2]\nonumber\\
& \leq \sqrt 3 \frac{1}{2\D (2\pi)^{2\d}} \nonumber\\
 & \quad \int_{\mathbb R^{\d}}  \frac{1}{ p^*_{\freqRV}(\freq)} |\hat \func(\freq)|^2 (\cos(2\angle \hat\func(\freq)+2\freq^T\input) + 2) d\freq \nonumber\\
& =
\sqrt 3 \ \mathbb{E}_{p^*_{\freqRV},p_{\B}}[\RFERV(\input)].
\end{align}
Combining~\eqref{eq:cosineq} with~\eqref{eq:minvarcomplex} yields:
\begin{align}
	 & \frac{1}{2} \mathbb{E}_{\tilde p_{\freqRV}}[\tilde \RFERV(\input)^2]  \nonumber\\
     & =  \frac{1}{2\D (2\pi)^{2\d}} \int_{\mathbb R^{\d}}  \frac{1}{ \tilde p_{\freqRV}(\freq)} |\hat \func(\freq)|^2 d\freq \nonumber\\
     & \leq \frac{1}{2\D (2\pi)^{2\d}} \int_{\mathbb R^{\d}}  \frac{1}{ \tilde p_{\freqRV}(\freq)} |\hat \func(\freq)|^2
     \nonumber\\ & \qquad
     (\cos(2\angle \hat\func(\freq)+2\freq^T\input) + 2) d\freq \nonumber\\
     & = \mathbb{E}_{\tilde p_{\freqRV},p_{\B}}[\RFERV(\input)^2] \nonumber\\
     & \leq \frac{3}{2\D (2\pi)^{2\d}} \int_{\mathbb R^{\d}}  \frac{1}{ \tilde p_{\freqRV}(\freq)} |\hat \func(\freq)|^2 d\freq \nonumber\\
     & =  \frac{3}{2} \mathbb{E}_{\tilde p_{\freqRV}}[\tilde \RFERV(\input)^2] .
\end{align}
\end{proof}

\section*{Acknowledgment}
This research was supported by the Netherlands Enterprise Agency (RVO) for Innovation in Photonic Devices (IPD12020), by the European Research Council Advanced Grant Agreement (No. 339681) and by the Dutch Technology Foundation STW (project 13336).

\bibliographystyle{IEEEtran}
\bibliography{mybib}





\end{document}